\documentclass{article}

\usepackage{microtype}
\usepackage{graphicx}
\usepackage{subfigure}
\usepackage{booktabs} %

\usepackage{hyperref}
\usepackage{nicefrac}       %

 \usepackage[accepted]{icml2023}

\usepackage{amsmath}
\usepackage{amssymb}
\usepackage{mathtools}
\usepackage{amsthm}

\usepackage[capitalize,noabbrev]{cleveref}

\theoremstyle{plain}
\newtheorem{theorem}{Theorem}[section]

\newtheorem{lemma}[theorem]{Lemma}

\theoremstyle{definition}

\newtheorem{assumption}[theorem]{Assumption}
\theoremstyle{remark}

\usepackage{xcolor}
\usepackage{pifont}
\usepackage{graphicx}
\definecolor{mydarkgreen}{RGB}{39,130,67}
\definecolor{mydarkred}{RGB}{192,25,25}
\newcommand{\green}{\color{mydarkgreen}}

\usepackage{multirow}
\usepackage{xspace}

\usepackage[textsize=tiny]{todonotes}

\graphicspath{{img/}}

\usepackage{amsopn}

\usepackage{microtype}
\usepackage{booktabs}
\usepackage{dsfont}
\usepackage{xcolor}
\usepackage{natbib}
\usepackage{graphicx}

\usepackage{algorithm}
\usepackage{algorithmic}

\renewcommand{\algorithmiccomment}[1]{\hfill //~\emph{#1}}

\usepackage[mathcal]{eucal}

\usepackage{amssymb}
\usepackage{amsfonts}

\usepackage{amsmath}
\usepackage{dsfont}

\newcommand{\esp}[1]{\mathbb{E}\left[ #1 \right]}
\newcommand{\E}{\mathbb{E}}

\newcommand{\R}{\mathbb{R}}

\definecolor{violett}{HTML}{FF00FF}

\newcommand{\cN}{\mathcal{N}}
\newcommand{\cT}{\mathcal{T}}
\newcommand{\cB}{\mathcal{B}}

\providecommand{\vg}{\mathbf{g}}

\newcommand{\Gxi}{\vg(\xx)}

\newcommand{\clip}{{\rm clip}_c}

\newcommand{\norm}[1]{\left\|#1\right\|}
\newcommand{\sqnorm}[1]{\norm{#1}^2}

\newcommand{\lzlone}{$(L_0, L_1)$\xspace}

\providecommand{\norm}[1]{\left\lVert#1\right\rVert}

\providecommand{\R}{\mathbb{R}} %

\providecommand{\E}[1]{{\mathbb E}\left.#1\right. }        %

\DeclareMathOperator*{\argmin}{arg\,min}

\renewcommand{\gg}{\mathbf{g}}

\providecommand{\uu}{\mathbf{u}}

\providecommand{\xx}{\mathbf{x}}
\providecommand{\yy}{\mathbf{y}}
\providecommand{\zz}{\mathbf{z}}

\providecommand{\mI}{\mathbf{I}}

\providecommand{\cB}{\mathcal{B}}

\providecommand{\cD}{\mathcal{D}}

\providecommand{\cN}{\mathcal{N}}
\providecommand{\cO}{\mathcal{O}}

\providecommand{\cT}{\mathcal{T}}

\usepackage{enumitem}
\renewcommand{\paragraph}[1]{\textbf{#1} }
\icmltitlerunning{Revisiting Gradient Clipping}

\begin{document}

\twocolumn[
\icmltitle{Revisiting Gradient Clipping: \\ Stochastic bias and tight convergence guarantees}

\icmlsetsymbol{equal}{*}

\begin{icmlauthorlist}
\icmlauthor{Anastasia Koloskova}{equal,epfl}
\icmlauthor{Hadrien Hendrikx}{equal,inria}
\icmlauthor{Sebastian U. Stich}{cispa}
\end{icmlauthorlist}

\icmlaffiliation{epfl}{EPFL, Switzerland}
\icmlaffiliation{inria}{Inria Grenoble, France (work done in part while at EPFL)}
\icmlaffiliation{cispa}{CISPA Helmholtz Center for Information Security, Germany}

\icmlcorrespondingauthor{Anastasia Koloskova}{anastasia.koloskova@epfl.ch}
\icmlcorrespondingauthor{Hadrien Hendrikx}{hadrien.hendrikx@inria.fr}

\icmlkeywords{SGD, clipping}

\vskip 0.3in
]

\printAffiliationsAndNotice{\icmlEqualContribution} %

\begin{abstract}
Gradient clipping is a popular modification to standard (stochastic) gradient descent, at every iteration limiting the gradient norm to a certain value $c >0$. It is widely used for example for stabilizing the training of deep learning models~\citep{Goodfellow2016:DL_and_clipping}, or for enforcing differential privacy~\citep{abadi16:dp-sgd}. Despite popularity and simplicity of the clipping mechanism, its convergence guarantees often require specific values of $c$ and strong noise assumptions. 

In this paper, we give convergence guarantees that show precise dependence on arbitrary clipping thresholds $c$ and show that our guarantees are tight with both deterministic and stochastic gradients. In particular, we show that (i) for deterministic gradient descent, the clipping threshold only affects the higher-order terms of convergence, (ii) in the stochastic setting convergence to the true optimum cannot be guaranteed under the standard noise assumption, even under arbitrary small step-sizes. We give matching upper and lower bounds for convergence of the gradient norm when running clipped SGD, and illustrate these results with experiments.
\end{abstract}

\section{Introduction}

This paper focuses on solving general minimization problem of the form
\begin{equation}\label{eq:prob}
    \min_{\xx \in \R^d} \left\{f(\xx) := \E_{\xi \sim \cD} [f_\xi(\xx)] \right\},
\end{equation}
where $f$ is a possibly non-convex, and possibly stochastic function. This setting covers many applications, e.g.\ it covers optimizing deterministic functions if $f_\xi \equiv f~\forall \xi$. It also covers minimizing the empirical loss in machine learning applications, where $\cD$ represents the uniform distribution over training datapoints, and $f_\xi(\xx)$ is the loss of model $\xx$ on the datapoint $\xi$. 

We focus on gradient descent methods with \emph{gradient clipping} for solving \eqref{eq:prob}. Given a clipping radius $c > 0$, step-size $\eta > 0$, and starting from a point $\xx_0 \in \R^d$ the gradient clipping algorithm performs the following iterations: 
\begin{align} \label{eq:clipped_sgd}
    \xx_{t+1} = \xx_t - \eta \vg_t\,, && \ \text{with}\ &&  \vg_t = \clip(\nabla f_\xi(\xx_t))\,,
\end{align}
where $\vg_t$ is a clipped stochastic gradient, and the clipping operator is defined  as
\begin{align}\label{eq:clipping}
    \clip(\uu) = \min\left(1, \frac{c}{\norm{\uu}}\right) \uu\,, \qquad \text{for}\ \uu \in \R^d.
\end{align}

Gradient clipping is widely used to \emph{stabilize} the training of neural networks, by preventing large occasional gradient values from harming it \cite{Goodfellow2016:DL_and_clipping}. This is particularly useful for mitigating outliers in the training data, and training recurrent models \cite{Pascanu2012:understand_clip, Pascanu13:rnn_clipping}, in which the noise can induce very large gradients.

Gradient clipping is also an essential part of privacy-preserving machine learning. The widely-used Gaussian Mechanism~\cite{Dwork14:DP_foundations} adds noise to the individual gradients to add uncertainty about their true value. Yet, it requires the gradients to have bounded norms for the privacy guarantees to hold. In practice, bounded gradients are enforced through clipping \cite{abadi16:dp-sgd}. 

Gradient clipping has already been widely studied, as we detail in the next section. However, many works choose a specific value for the clipping threshold $c$ in order to guarantee convergence. This suggests that $c$ should be carefully tuned in practice, which is highly undesirable, in particular since the clipping threshold might be dictated by other (e.g., privacy) concerns. Besides, most works impose strong assumptions on the stochastic gradient noise, through either large batches (and thus small stochastic noise), angle conditions, or uniform boundedness of the norm, that might not hold in practice.

In this work, we precisely characterize how the clipping threshold $c$ affects the convergence properties of clipped-SGD for \emph{any} clipping threshold $c$. %

We consider deterministic and stochastic functions separately, as clipping affects these two settings in different ways. 

In the \emph{deterministic case}, clipping only changes the magnitude of the applied gradients, but not their direction. This means that clipped gradient descent can reach the critical points of $f$, however slower. Intuitively, as the algorithm converges, the gradients become small in magnitude and are not clipped eventually. This means that clipping affects only the speed during the first phase when the gradients are large in magnitude. The main challenge is to tightly characterize this overhead. 

In the \emph{stochastic case}, the story is different:\ the individual stochastic gradients can be large even though the expected gradient is small. Even at the critical points of $f$, where the expected (full) gradient is zero, there is some probability that individual stochastic gradients are clipped. Moreover, as we do not assume any symmetry of the stochastic gradients, the expected clipped gradient might be non-zero even at critical points of $f$, forcing the algorithm to drift from these critical points. 
The direct consequence of this is that clipped SGD \emph{does not converge to the critical points of $f$ in general} \cite{chen20:geometric_clipping}, but only to some neighborhood. In this paper we study the bias introduced by clipping and show that it depends on the noise variance $\sigma^2$ and the clipping parameter $c$, that we precisely define in Section~\ref{sec:def}. As we will further detail in Section~\ref{sec:related_work}, existing works circumvent this difficulty either by using large clipping thresholds or large mini-batches, or by making strong assumptions on the noise such as uniform boundness, restricted angles between stochastic gradients, etc., and usually requiring specific values for $c$. Instead, we tightly analyze the convergence of clipped SGD and \emph{characterize precisely} the bias introduced by clipping without any additional assumptions. 

More specifically, our contributions are the following:
\begin{itemize}[leftmargin=12pt,itemsep=1pt]
    \item In the deterministic setting, we analyze the convergence behavior of clipped gradient descent for non-convex, convex and strongly convex functions. Our analysis shows that in all the cases after some transient regime, clipping does not affect the convergence rate. This initial phase does not affect the leading term of convergence in the convex and non-convex cases. However, in the strongly convex case, this unavoidable initial phase does not ensure linear forgetting of the initial conditions, resulting in a substantial slowdown. 
    \item For stochastic gradients, we show that clipped SGD under the `heavy-tailed' assumption converges to a neighbourhood of size $\min\{\sigma, \sigma^2 / c\}$, measured in terms of the gradient norm.
    \item We show that this neighborhood size is tight: clipped SGD reduces the gradient norm up to $\min\{\sigma, \sigma^2 / c\}$ indeed, provided the step-size is small enough.
    \item We frame our results using the \lzlone-smoothness assumption~\citep{zhang19:clippingL0-L1}, a standard relaxation of smoothness that is well suited to analyzing clipped algorithms.
\end{itemize}

Through these results, we aim at painting a thorough and accurate landscape of the convergence guarantees of clipping \emph{under the same assumptions as standard SGD, and for any clipping threshold $c$}. Our goal is that these improved bounds will allow to tighten guarantees for all downstream applications, e.g.\ privacy, that use clipped-SGD convergence results as black box. Indeed, the clipping threshold is often viewed as an external parameter of the problem in these cases, whereas our flexible guarantees allow to optimize the bounds for $c$ and trade-off convergence speed (or precision) and application-specific requirements.

\begin{table*}
    \caption{Comparison of key assumptions and illustration of complexity estimates in the non-convex deterministic case (variance $\sigma=0$).}
    \label{tab:relatedwork}
    \resizebox{\textwidth}{!}{
    \centering
    \begin{tabular}{llllllr}
        \toprule[1pt]
        reference       & smoothness  & variance bound &  clipping threshold & further assumptions   & \multicolumn{2}{l}{rate (non-convex, $\sigma=0$)}\\
        \midrule 
        \citet{zhang19:clippingL0-L1}  & 2nd-order \lzlone & uniform bd. & $c = \Theta(\min\{L_0, \frac{L_0}{L_1}\})$ & & %
        $\cO\left(\frac{1}{\sqrt{\eta T c}}\right)$ & \hspace{-4cm} $\eta\leq \min \Big\{ \frac{1}{10 L_0}, \frac{1}{10 cL_1} \Big\}$  %
        \\
         \citet{zhang2020improved}     & {\green \lzlone } & uniform bd. & $c = \Theta(\max\{\epsilon, \frac{L_0}{L_1}\})$ & & $\cO\left(\frac{1}{\sqrt{\eta T}} \right) $ &  \hspace{-4cm}  $\eta\leq \frac{1}{10 L_0} $%
         \\ \midrule
        \citet{chen20:geometric_clipping} &  $L$ &  {\green expectation} &  {\green arbitrary} & pos. skewness &  $\cO\left(\frac{1}{\sqrt{\eta T}} + \frac{1}{\eta T c} + \sqrt{\eta L c^2} \right)$ & \hspace{-4cm}  \\
        \citet{qian21:understanding_clipping} &  {\green \lzlone }&  {\green expectation} &  {\green arbitrary} & pos. alignment &  $\cO\left(\frac{1}{\sqrt{\eta T}} + \frac{1}{\eta T c} + \sqrt{\eta L_0 c^2} + \sqrt{\eta L_1 c^3}\right) $ & \hspace{-4cm} $\eta \leq \frac{1}{4 c L_1}$\\
        \textbf{ours}   &  {\green \lzlone } & {\green expectation} & {\green arbitrary} &  & {\green $\cO \left(\frac{1}{\sqrt{\eta T}} + \frac{1}{\eta T c}\right)$} & \hspace{-1cm} {\green $\eta \leq \frac{1}{9(L_0 + c L_1)}$} \\
        \bottomrule    
    \end{tabular}
    }
\end{table*}

\subsection{Main assumptions}
\label{sec:def}
Before discussing related work we will first state the assumptions we use in our work. 

\paragraph{Assumption on smoothness.}
The widely used smoothness assumption in the optimization literature (e.g.~\citealp{nesterov2018lectures}) is the following:
\begin{assumption}[Smoothness]\label{a:smooth-classic} Function $f$ satisfies 
	\begin{align*}
	\norm{\nabla f(\xx) - \nabla f(\yy)} \leq L \norm{\xx - \yy} \,, && \forall \xx, \yy \in \R^d \,.
	\end{align*}
\end{assumption}
Despite its widespread use, this assumption can be restrictive, as the constant $L$ must capture the worst-case smoothness. 
\citet{zhang19:clippingL0-L1} experimentally discovered that for various deep learning tasks, the local smoothness constant $L$ decreases during training, and is proportional to the gradient norm. They reported that the local curvature (smoothness) in the final stages of training could be $1000$ times smaller than the curvature at the initialization point (for LSTM training on the PTB dataset).
\lzlone-smoothness \cite{zhang19:clippingL0-L1, zhang2020improved} has been proposed as a natural relaxation of the classical smoothness assumption.
\begin{assumption}[\lzlone-smoothness]\label{a:smooth}
	A differentiable function $f \colon \R^d \to \R$ is said to be \lzlone-smooth if it verifies for all $\xx, \yy \in \R^d$ with $\norm{\xx-\yy} \leq \frac{1}{L_1}$:
	\begin{equation}
	\norm{\nabla f(\xx) - \nabla f(\yy)} \leq (L_0 + \norm{\nabla f(\xx)} L_1) \norm{\xx - \yy} \,. \label{eq:424}
	\end{equation}
\end{assumption}
We use this as the main assumption in our work. This assumption recovers the standard smoothness Assumption~\ref{a:smooth-classic} by taking $L_1 = 0$. However, taking $L_1 > 0$ allows to obtain smooth-like properties for functions that would otherwise not be smooth, such as $\xx \mapsto \norm{\xx}^3$.
Moreover, it is possible that $L$-smooth functions are \lzlone-smooth with both of the constants $L_0, L_1$ significantly smaller than $L$, such as for the exponential function $x \mapsto e^x$.

Note that the imposed bound $\norm{\xx-\yy}\leq \frac{1}{L_1}$ in \eqref{eq:424} is essential, as otherwise the global growth of the gradients would be similarly restricted as for standard smooth functions (thereby excluding functions such as the mentioned $\xx \mapsto \norm{\xx}^3$).

In their work on clipping algorithms, \citet{zhang19:clippingL0-L1} used a slightly stronger smoothness condition that required second-order differentiability. For twice-differentiable functions $f$, they defined \lzlone-smoothness as
\begin{equation}
\norm{\nabla^2 f(\xx)} \leq L_0 + L_1 \norm{\nabla f(\xx)}\,, \qquad \forall \xx \in \R^d. \label{eq:secondorder}
\end{equation}
Later, \citet{zhang2020improved} noticed that the weaker Assumption~\ref{a:smooth} is sufficient for the study of clipping algorithms. We adopt their notion in our work.

\paragraph{Assumption on stochastic variance.}
Many works in clipping literature \citep{zhang19:clippingL0-L1,zhang2020improved, yang22:normalized_and_clipped_for_dp} assume the following 
\begin{assumption}[Uniform boundness]\label{def:unif_var}
	We say that the stochastic noise of $f_\xi$ is uniformly bounded by $\sigma^2$ if for all $\xx \in \R^d$, %
	\begin{equation}
	\Pr\left[\sqnorm{\nabla f_\xi(\xx) - \nabla f(\xx)} \leq \sigma^2\right] = 1.
    \label{eq:unif_var}
	\end{equation}
\end{assumption}
While this assumption allows to simplify the analysis of clipping algorithms, this is a very strong assumption. 

By making Assumption~\ref{def:unif_var} and using a sufficiently large clipping radius ($c > \sigma$), we can guarantee that stochastic gradients are not clipped at the critical points of $f$ where $\nabla f(\xx) = 0$. This ensures that the algorithm can converge to the exact critical points, simplifying the theoretical analysis in~\citep{zhang19:clippingL0-L1,zhang2020improved, yang22:normalized_and_clipped_for_dp}.

The uniform boundness Assumption~\ref{def:unif_var} is a strong assumption and may not always be reflective of the real-world scenarios. For instance, if gradients are perturbed by Gaussian noise, the assumption of uniform boundness does not hold. Additionally, in machine learning applications where $\nabla f_\xi(\xx)$ represents gradients of a model $\xx$ at different datapoints $\xi$ from a dataset $\xi \in \cD$, a uniform bound on $\sigma$ may be large if the dataset $\cD$ has even only one outlier point.

In this work, we use the weaker and the more standard variance definition instead \citep{Lan12:stochastic_opt, Dekel12:minibatch},
 sometimes called \emph{heavy tailed noise}~\citep{gorbunov20:acc_clipping}. %
\begin{assumption}[Bounded variance] \label{def:var}
	We say that the variance of $f_\xi$ is bounded %
    by $\sigma^2$ if for all $\xx \in \R^d$
	\begin{equation}
	\esp{\sqnorm{\nabla f_\xi(\xx) - \nabla f(\xx)}} \leq \sigma^2.
	\end{equation}
\end{assumption}
Note that uniform boundedness implies bounded variance (with the same constant), but not the other way round. 

\subsection{Related work}
\label{sec:related_work}

The literature on gradient clipping is already extensive and still very active. We present the most relevant contributions for our work below and display a selection in Table~\ref{tab:relatedwork}.

\paragraph{Clipping stabilizes learning.} Gradient clipping was originally proposed in \cite{Mikolov12:thesis_with_clipping} in order to tackle the gradient explosion problem in training of recurrent neural networks. \citet{zhang19:clippingL0-L1} proposed to theoretically explain the question why clipped SGD improves the stability of (stochastic) first-order methods, by imposing a relaxed second-order \lzlone-smoothness assumption (see Equation~\eqref{eq:secondorder}), and showing the convergence advantages of clipped SGD over unclipped SGD. However they rely on the strong Assumption~\ref{def:unif_var} for the stochastic variance and chose the clipping threshold to a specific large enough value. %
 The favorable convergence guarantees were then refined by~\citet{zhang2020improved}, while still relying on Assumption~\ref{def:unif_var} on the stochastic noise and choosing specific value for the clipping threshold. %
 \citet{vien21:stability_and_converg_of_clip} show that this is also the case in the non-smooth setting. %

\paragraph{Noise assumptions.} Gradient clipping is often analyzed under uniform boundness Assumption~\ref{def:unif_var} on the stochastic noise of the gradients in combination with choosing a large enough clipping threshold $c > \sigma$%
~\citep{zhang19:clippingL0-L1,zhang2020improved, yang22:normalized_and_clipped_for_dp}. Choosing large enough values of $c$ simplifies the theoretical analysis. However, in some applications the choice of the clipping threshold $c$ might be dictated by other constraints, such as privacy constraints. Especially because in many practical applications the stochastic noise is heavy-tailed \cite{Jingzhao20:adaptive_methods} it would entail large values of $\sigma$, and thus $c$.

To avoid the uniformly bounded noise assumption, some works impose other strong assumptions on the distribution of stochastic gradients. For instance, \citet{qian21:understanding_clipping} restrict the angle between stochastic gradients and the true gradient, and \citet{chen20:geometric_clipping} impose a symmetry assumption on the distribution of the stochastic gradients.
\citet{gorbunov20:acc_clipping} analyze clipping under bounded variance (see Assumption~\ref{def:var}), however, they impose a strong assumption of the size of the minibatches used to scale linearly with $T$, thus making the effective stochastic variance to be diminishing with the number of iterations $T$ as $\cO\bigl(\frac{\sigma^2}{T}\bigr)$.

In this paper we take a different route from all these works, and analyse clipped SGD under the much weaker \emph{bounded variance} assumption. Yet, instead of converging to the exact critical points of $f$, we quantify how large the drift due to clipping is, and thus obtain guarantees for any values of the batch size and $c$. 

\paragraph{Noiseless case.} Since the bulk of the assumptions concern the stochastic noise, our setting is the same as the papers mentioned above in the deterministic setting. However, in this case, we give sharper guarantees, essentially proving that clipping does not affect the leading convergence terms (see Table~\ref{tab:relatedwork}).

\paragraph{Clipped federated averaging.} \citet{zhang22:understanding_clip_federated} study clipping for the FedAvg \cite{McMahan2016:FedAvg} algorithm, by clipping the model differences sent to the server. However, bounded gradients are needed, and the convergence rate does not recover the rate of FedAvg when the clipping threshold $c \to \infty$. Moreover, clipped FedAvg is biased even when using deterministic gradients. \citet{Mingrui22:effifient_clipping} also study a clipped-FedAvg-like algorithm, and get rid of bias issues through assuming symmetric noise distributions around their means.  %

\paragraph{Differentially private SGD.} Differential privacy has become the gold standard for protecting privacy, thus raising interest from the stochastic optimization community~\citep{chaudhuri2011differentially,song2013stochastic, duchi2014privacy}. However, to ensure differential privacy, boundedness of the stochastic gradients~\citep{wang2017differentially,bassily2019private, das2022beyond} (or a related condition, such as Lipschitzness of the objective function) has to hold. This is rarely true in practice, but instead enforced via clipping, such as in the DP-SGD algorithm~\citep{abadi16:dp-sgd}. Indeed, Lipschitzness requires the gradients to be bounded, whereas smoothness only requires boundedness of the Hessian. Although smoothness implies Lipschitzness on a bounded domain, this bound is usually very conservative and leads to poor guarantees.

\citet{Bagdasaryan19:DP-for-DL-bad-effect} experimentally measure the effect of DP-SGD (clipping and additional noise) on model accuracy. They observe that the gradients do not converge to zero norm, so that the assumptions under which exact convergence is shown are often not verified indeed. Besides, underrepresented classes have higher gradient norm (so DP-SGD affects fairness).

\paragraph{Connection to adaptive methods. } It is worth noting that clipped SGD is related to adaptive methods, such as the Adam algorithm \cite{KingmaB14:Adam}, or normalized SGD \cite{Hazan15:norm,Kfir16:norm}, that also perform a scaling of the gradient.
However, 
these two algorithms are not equivalent to clipped SGD and the convergence results for the Adam algorithm~\cite{Reddi2018:adam, Jingzhao20:adaptive_methods, Zhang2022AdamCC} and normalized SGD \cite{Zhao2021:normalized} cannot be directly translated to the clipped SGD.

\section{Deterministic Setting}

In this section we consider gradient clipping algorithm \eqref{eq:clipped_sgd} with full (deterministic) gradients, i.e.\ with 
\begin{align}\label{eq:deterministic_grad}
\nabla f_\xi(\xx) \equiv \nabla f(\xx)\,, && \forall \xi \in \cD, \forall \xx \in \R^d.
\end{align}
In this setting, the clipping operator \eqref{eq:clipping} only changes the magnitude of the applied gradients, without changing its direction (as opposed to taking the expectation of clipped stochastic gradients). Thus, we can expect convergence to the exact minima, resp.\ critical points, of the function $f$. 
It still remains unclear how much does such a change in the magnitude of the gradients affect the convergence speed of the algorithm. %

In our theoretical results we show that the drastic slow down happens \emph{only} if the function $f$ is \emph{strongly convex}, in which case the initial conditions (distance to optimum) are not forgotten linearly anymore once clipping is applied. However, the leading term in the error $\epsilon$ is unaffected. If the function $f$ is either \emph{convex} or \emph{non-convex}, the clipping threshold $c$ does not affect the leading term of convergence, and affects \emph{only the higher-order terms}.

\subsection{Non-convex functions}
\begin{theorem}[non-convex]\label{thm:det_nc}
    If $f$ satisfies Assumption~\ref{a:smooth}, then clipped gradient descent \eqref{eq:clipped_sgd} with deterministic gradients \eqref{eq:deterministic_grad} and with stepsize $\eta \leq [9(L_0 + c L_1)]^{-1}$ guarantees an error:
    \begin{equation}
        \frac{1}{T}\sum_{t=1}^T \norm{\nabla f(\xx_t)}\leq \cO\left(\sqrt{\frac{F_0}{\eta T}} + \frac{F_0}{\eta T c}\right), \label{eq:gd_nonsquare}
    \end{equation}
    where $T$ is the number of iterations, $F_0 = f(\xx_0) - f^\star$.
\end{theorem}
This theorem is a consequence of Theorem~\ref{thm:csgd_large_c} for $\sigma = 0$. %

\paragraph{Comparison to the unclipped gradient descent.}
The convergence rate of gradient descent (without clipping) assuming the standard $L$-smoothness Assumption~\ref{a:smooth-classic} is equal to \citep{ghadimi2013stochastic}:
\begin{align}
\frac{1}{T} \sum_{t = 1}^T \norm{\nabla f(\xx_t)}^2 \leq \cO \left(\frac{F_0}{\eta T}\right), \label{eq:gd_square}
\end{align}
where the stepsize must be smaller than $\eta \leq \frac{1}{L}$. In the future discussion we will assume that $L_0 + cL_1 \leq L$, as we can always choose $L_1$ to be zero. In many cases, both $L_0$ and $L_1$ are significantly smaller than $L$ (as discussed in Section~\ref{sec:def}).
Thus, compared to the unclipped gradient descent, clipped gradient descent \eqref{eq:clipped_sgd}:
\begin{itemize}[leftmargin=12pt,nosep,labelwidth=18pt,itemindent=6pt]
	\item[(i)] allows for larger stepsizes $\eta$ (up to the constant 9 in the stepsize constraint).
	This result is due to the refined \lzlone smoothness assumption and such an improvement in the stepsize has the same spirit as the discovery made by \citet{zhang19:clippingL0-L1} for the \lzlone second-order smoothness assumption~\eqref{eq:secondorder}, although their bound on the stepsize is different.
	\item[(ii)] has an additional term $\frac{F_0}{\eta T c}$ that depends on the clipping radius $c$. This term is of the order $\frac{1}{T}$, while the leading (the slowest decreasing, asymptotically dominating) term is of order $\frac{1}{\sqrt{T}}$. 
	
	If $c$ is small, this term will slow down the algorithm significantly. However, when $c$ is chosen larger than the final target accuracy $\epsilon$, clipping affects the convergence speed only by a constant factor.
	Intuitively, this is because the number of steps when clipping happens is only a constant fraction of the total required number of iterations to converge.\footnote{%
	Formally: because the final accuracy $\epsilon = \sqrt{\nicefrac{F_0}{\eta T}} + \nicefrac{F_0}{\eta T c} \geq \sqrt{\nicefrac{F_0}{\eta T}}$ , and thus if clipping threshold is larger than that, $c \geq \sqrt{\nicefrac{F_0}{\eta T}}$, then the convergence speed $\sqrt{\nicefrac{F_0}{\eta T}} + \nicefrac{F_0}{\eta T c} \leq 2 \sqrt{\nicefrac{F_0}{\eta T}} $ is affected only by a constant.
	}
	 As we frequently know the final target accuracy, our result shows that the clipping threshold could be set to avoid the adversarial effect of clipping. However, in practice the clipping threshold might be dictated by other needs. %
	\item[(iii)] has the different convergence measure $\frac{1}{T}\sum_{t=1}^T \norm{\nabla f(\xx_t)}$ instead of $\frac{1}{T} \sum_{t = 1}^T \norm{\nabla f(\xx_t)}^2$ that is more commonly used (e.g. in~\eqref{eq:gd_square} for unclipped gradient descent). %

\end{itemize}

\paragraph{Comparison to the prior work.} We summarized differences to the prior works in Table~\ref{tab:relatedwork}. \citet{zhang19:clippingL0-L1} and \citet{zhang2020improved} analyzed deterministic gradient clipping however setting the clipping threshold $c$ to some specific, large enough values.
\citet{qian21:understanding_clipping} and \citet{chen20:geometric_clipping} analyse clipped SGD under arbitrary choice of the clipping threshold $c$. In particular, assuming deterministic gradients ($\sigma = 0$), 
\citet{qian21:understanding_clipping} obtain the convergence rate of $\cO\left(\sqrt{\frac{F_0}{\eta T}} + \frac{F_0}{\eta T c} + c\sqrt{\eta L_0} + c^{3/2}\sqrt{\eta L_1}\right)$, $\eta < \nicefrac{1}{4 c L_1}$, as the two last terms $ c\sqrt{\eta L_0} + c^{3/2}\sqrt{\eta L_1}$ do not decrease to zero under the constant stepsizes $\eta$. %
 That is strictly worse than ours in Theorem~\ref{thm:det_nc}. \citet{chen20:geometric_clipping} prove the rate $\cO\left(\sqrt{\frac{F_0}{\eta T}} + \frac{F_0}{\eta T c} + c\sqrt{\eta L} \right)$ without any constraint on the stepsize, however they have to take small stepsizes $\eta = \nicefrac{1}{\sqrt{T}}$ as the term $c\sqrt{\eta L}$ is not decreasing in $T$. Notably, this terms prevents the error from converging to $0$ under constant step-sizes, which can be obtained in the deterministic setting, as we showed above. %

\subsection{Convex functions}
We now prove an equivalent theorem when $f$ is convex, i.e.\ assuming additionally:
\begin{assumption}[Convexity]\label{a:convex} Function $f$ satisfies 
	\begin{align*}
	f(\xx) - f(\yy) \leq \langle \nabla f(\xx), \xx - \yy\rangle\,, && \forall \xx, \yy \in \R^d \,.
	\end{align*}
	We also assume that infimum of $f$ is achieved in $\R^d$. 
\end{assumption}
\begin{theorem}[convex]\label{thm:det_cvx}
	If $f$ is $L$-smooth\footnote{We can relax this assumption. Equation~\eqref{eq:thm23} also holds with $L$ replaced by $L_T$, defined as $L_T := \max_{t \leq T }\{L_0 + L_1\norm{\nabla f(\xx_t)} \}$.} (Assumption~\ref{a:smooth-classic}), \lzlone smooth (Assumption~\ref{a:smooth}) and convex (Assumption~\ref{a:convex}), then clipped gradient descent \eqref{eq:clipped_sgd} with deterministic gradients \eqref{eq:deterministic_grad} and with stepsize $\eta \leq (L_0 + c L_1)^{-1}$ guarantees an error:
    \begin{equation}
        f(\xx_T) - f^\star \leq \cO \left(\frac{R_0^2}{\eta T } + \frac{R_0^4 L}{\eta^2 T^2 c^2} \right)\,, \label{eq:thm23}
    \end{equation}
    where $R_0^2 = \sqnorm{\xx_0 - \xx_\star}$, $f^\star = f(\xx_\star)$, and $\xx^\star~=~\argmin_{\xx} f(\xx)$. %
\end{theorem}

In comparison, under the same assumptions as in Theorem~\ref{thm:det_cvx}, unclipped gradient descent converges at the rate
\begin{align*}
f(\xx_T) - f^\star \leq \cO\left(\frac{R_0^2}{\eta T }\right)
\end{align*}
under the condition that the stepsize is smaller than $\eta \leq L^{-1}$~\citep{nesterov2018lectures}. Similarly to the non-convex case, the convergence rate of the clipped gradient descent is slowed down by the higher-order term $\nicefrac{R_0^4 L}{\eta^2 T^2 c^2}$. Yet, again, it is enough to set the clipping threshold $c$ bigger than the final target accuracy $\epsilon$ (multiplied by  $\sqrt{L}$ this time) to avoid the slowdown effect of this term, since for high accuracies more time is actually spent using unclipped gradients.

Also, similarly to the non-convex case, clipped GD allows for the larger stepsizes $\eta$ that would result in the faster convergence. 

\subsection{Strongly convex functions}
In this section we consider strongly-convex functions $f$.

\begin{assumption}[strong-convexity]\label{a:strong-convex} There exists a constant $\mu > 0$ such that function $f$ satisfies for all $ \xx, \yy \in \R^d$,
	\begin{align*}
	f(\xx) - f(\yy) + \frac{\mu}{2} \norm{\xx - \yy}_2^2 \leq \langle \nabla f(\xx), \xx - \yy\rangle \,. %
	\end{align*}
\end{assumption}
Similarly to the convex and non-convex cases, clipping does not affect the leading term of convergence (as $\epsilon \rightarrow 0$) as we show in the following theorem. This is due to the fact that for any fixed $c>0$, gradients are eventually never clipped.

\begin{theorem}[Strongly convex case]
\label{thm:det_scvx}
If $f$ is $\mu$-strongly convex  (Assumption~\ref{a:strong-convex}), $L$-smooth\footnote{We can relax this assumption by defining instead $L := \max_{t \leq T }\{L_0 + L_1\norm{\nabla f(\xx_t)} \}$.} (Assumption~\ref{a:smooth-classic}) and \lzlone smooth (Assumption~\ref{a:smooth}), then clipped gradient descent \eqref{eq:clipped_sgd} with deterministic gradients \eqref{eq:deterministic_grad} and with stepsize $\eta \leq (L_0 + c L_1)^{-1}$ needs at most
\begin{align} \hskip-1ex
 T =\cO\left(\frac{1}{\mu \eta} \log \left( \frac{R_0^2}{\epsilon} \right) + \frac{R_0 }{c \eta} \min\left(\sqrt{\frac{L}{\mu}}, \frac{L R_0}{c}\right) \right)
\end{align}
iterations to reach accuracy $R_T^2 \leq \epsilon$,  where $R_t^2 = \norm{\xx_t - \xx_\star}^2$ and  $\xx^\star~=~\argmin_{\xx} f(\xx)$.
\end{theorem}

Compared to the unclipped case, there is an extra term in the strongly convex case, which does not decrease with $\epsilon$ and corresponds to the overhead of clipping. This means that during the initial phase of convergence, when the gradients are clipped, the convergence speed is sublinear, and one would have to set $c = \cO\left(\nicefrac{1}{\log(\frac{1}{\epsilon})}\right)$ in order for the clipping do not affect the convergence speed, that is much larger than in the non-convex and convex cases. 

Intuitively, since the clipped gradient norm is fixed, the iterates can actually move only up to $\eta c$ each step towards the optimum. If the initial distance to the optimum $R_0$ happened to be large, in the best case scenario, one would need at least $\frac{R_0}{\eta c}$ steps to reach the optimum. The dependency on $c$ for the first term in the ${\rm min}$ is tight.\footnote{
Formally, consider the function $x \mapsto \frac{1}{2}x^2$ and initial iterate $x_0=1$. Suppose we aim to reach a target accuracy $\epsilon < \frac{1}{4}$ with clipping threshold $c < \frac{1}{2}$. We see that unless $|x|<c$, the gradient $f'(x)=x$ will get clipped to value $c$, and hence after $\frac{1}{2c}$ iterations, cannot reach a point with norm smaller than $\frac{1}{2}$ and squared norm less than $\frac{1}{4}$ respectively.}

Note that after a constant (independent of $\epsilon$) number of iterations, the algorithm converges linearly, at a rate that depends on $(L_0 + c L_1)$ only, that can be significantly smaller than the  dependency on $L$ in the GD convergence rate.

\section{Stochastic Functions}

In the deterministic setting gradient clipping achieves convergence to the exact minimizer or a stationary point, respectively, and clipping only affects initial convergence speed. Yet, this does not hold in the stochastic setting where clipping introduces \emph{unavoidable bias}. The main reason behind this is that the expectation of the clipped stochastic gradients is different (in both norm and direction) from the clipped true gradient.

While it was known before that the clipped SGD does not converge under the bounded variance Assumption~\ref{def:var} \cite{chen20:geometric_clipping}, in this section we will \emph{precisely characterize} lower bounds on the error that clipped SGD can achieve, and then we provide upper bounds that match our lower bounds.

\subsection{Unavoidable bias introduced by clipping}\label{sec:unavoidable}

If the gradient clipping algorithm converges, it has to be towards its fixed points, i.e.\ to points $\xx^\star$ such that $ \esp{\clip(\nabla f_\xi(\xx^\star))} = 0$. This is achieved in the limit of small step-sizes (to counter stochastic noise).

We now formally show a lower bound that states that fixed points of clipped SGD are not necessary optimal or critical points of the objective function $f$. In fact, there exist stochastic gradient noise distributions under which critical points of clipped SGD are $\sigma$ far away from critical points of~$f$. 

\begin{theorem}[Small clipping radius]\label{thm:lb_small_c}
	We fix a class of functions that have variance at most $\sigma^2$ (Def~\ref{def:var}) and smoothness parameters $L_0 = 1, L_1 = 0$ (Assumption~\ref{a:smooth}). Then, for any clipping threshold $c \leq 2 \sigma$, we can find a function $f$ within this fixed class such that the fixed points of clipped-SGD exist (i.e. points $\xx^\star$  which verify $\esp{\clip(\nabla f_\xi(\xx^\star))} = 0$), and that for all such fixed-points $\xx^\star$ of clipped-SGD it holds that $\norm{\nabla f(\xx^\star)} \geq \sigma / 12$.
\end{theorem}

\begin{proof}[Proof sketch.]
    We define the stochastic function 
    \begin{align*}
    f_\xi(x)  =\frac{1}{2} \begin{cases}
    (x + a)^2 & \text{w. p. }~ p \\
    x^2  & \text{w. p. }~ (1-p)
    \end{cases}
    \end{align*}
    where $a > 0$ and $p < 1/2$. The expected function is thus $f(x) = \frac{1}{2} [p (x + a)^2 +  (1 - p) x^2 ]$. 
    
   The result is then obtained by choosing $a = 4\sigma$ and $p = (2 - \sqrt{3})/4 <  1/4$ is such that $p(1 - p) = 1/16$. The statement follows by using the standard algebra, as detailed in Appendix~\ref{app:lower_bound}. 
\end{proof}

The previous impossibility result holds when the clipping radius is small ($c < 2 \sigma$). We further show that by taking a larger clipping radius $c$, we can reduce the neighborhood size to which clipped-SGD converges from $\sigma$ to $\nicefrac{\sigma^2}{c}$, but cannot completely eliminate it.

\begin{theorem}[Large clipping radius] \label{thm:lb_large_c}
	We fix a class of functions that have variance at most $\sigma^2$ (Def~\ref{def:var}) and smoothness parameters $L_0 = 1, L_1 = 0$ (Assumption~\ref{a:smooth}). Then, for any clipping threshold $c \leq 2 \sigma$, we can find a function $f$ within this fixed class such that the fixed points $\xx^\star$ of clipped-SGD exist ($\esp{\clip(\nabla f_\xi(\xx^\star))} = 0$) and $\norm{\nabla f(\xx^\star)} \geq \nicefrac{\sigma^2}{6c}$.
\end{theorem}

\begin{proof}[Proof sketch.]
    We use the same function as Theorem~\ref{thm:lb_small_c}, this time with $a = 2c$ and $p(1 - p) = \sigma^2 / a^2$.
\end{proof}

Our lower bounds in Theorems~\ref{thm:lb_small_c} and \ref{thm:lb_large_c} mean that with only assuming Def.~\ref{def:var} and Assumption~\ref{a:smooth}, there cannot exist problem-dependent values for $c$ that would give exact convergence for any function. We note that the fact that clipping might introduce a bias when the noise is only bounded in expectation is not new~\citep{chen20:geometric_clipping}. The interesting thing about Theorems~\ref{thm:lb_small_c} and \ref{thm:lb_large_c} is that they are matched by the upper bound in Theorem~\ref{thm:csgd_large_c}, meaning that we \emph{precisely capture} the strength of the bias introduced in this case.

\paragraph{Uniformly bounded noise.} Note that this lower bound crucially relies on the noise being bounded by $\sigma$ \emph{in expectation} (Assumption~\ref{def:var}). Indeed, the results hinge on the fact that stochastic gradients are clipped with probability $p$, thus introducing a bias. If we keep this Bernoulli noise constant (and therefore will ensure uniformly bounded noise Assumption~\ref{def:unif_var}) and increase $c$, then this bias would completely disappear for $c \approx a$, because then the clipping radius would be larger than the uniform bound on variance. %

\subsection{Convergence results}

\begin{figure*}[htb]
	\centering     %
	\subfigure[Constant stepsize, target $\epsilon = 10^{-2}$]{\label{fig:det1}\includegraphics[width=0.32\linewidth]{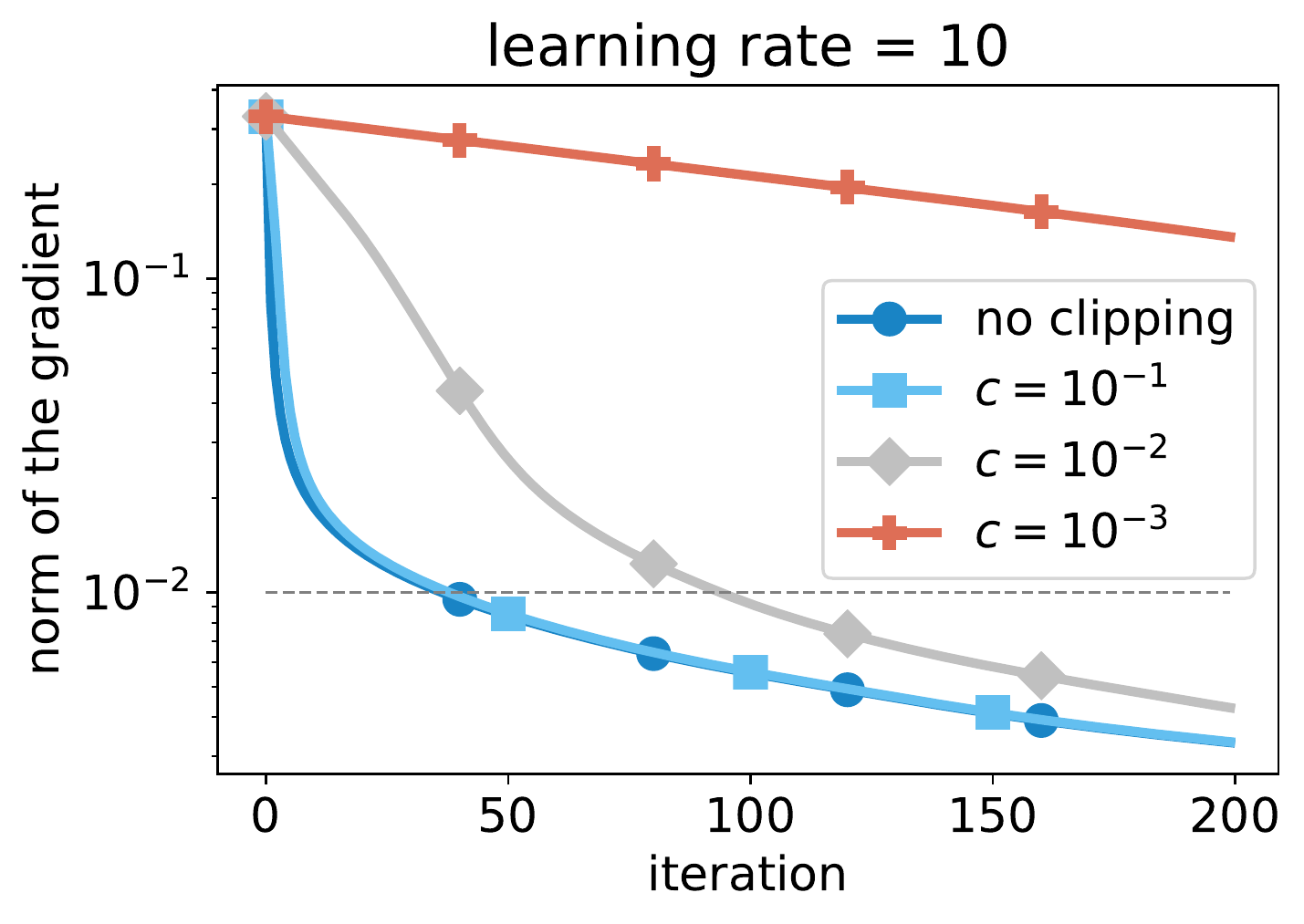}}%
	\subfigure[Constant stepsize, target $\epsilon = 10^{-3}$]{\label{fig:det2}\hfill \includegraphics[width=0.32\linewidth]{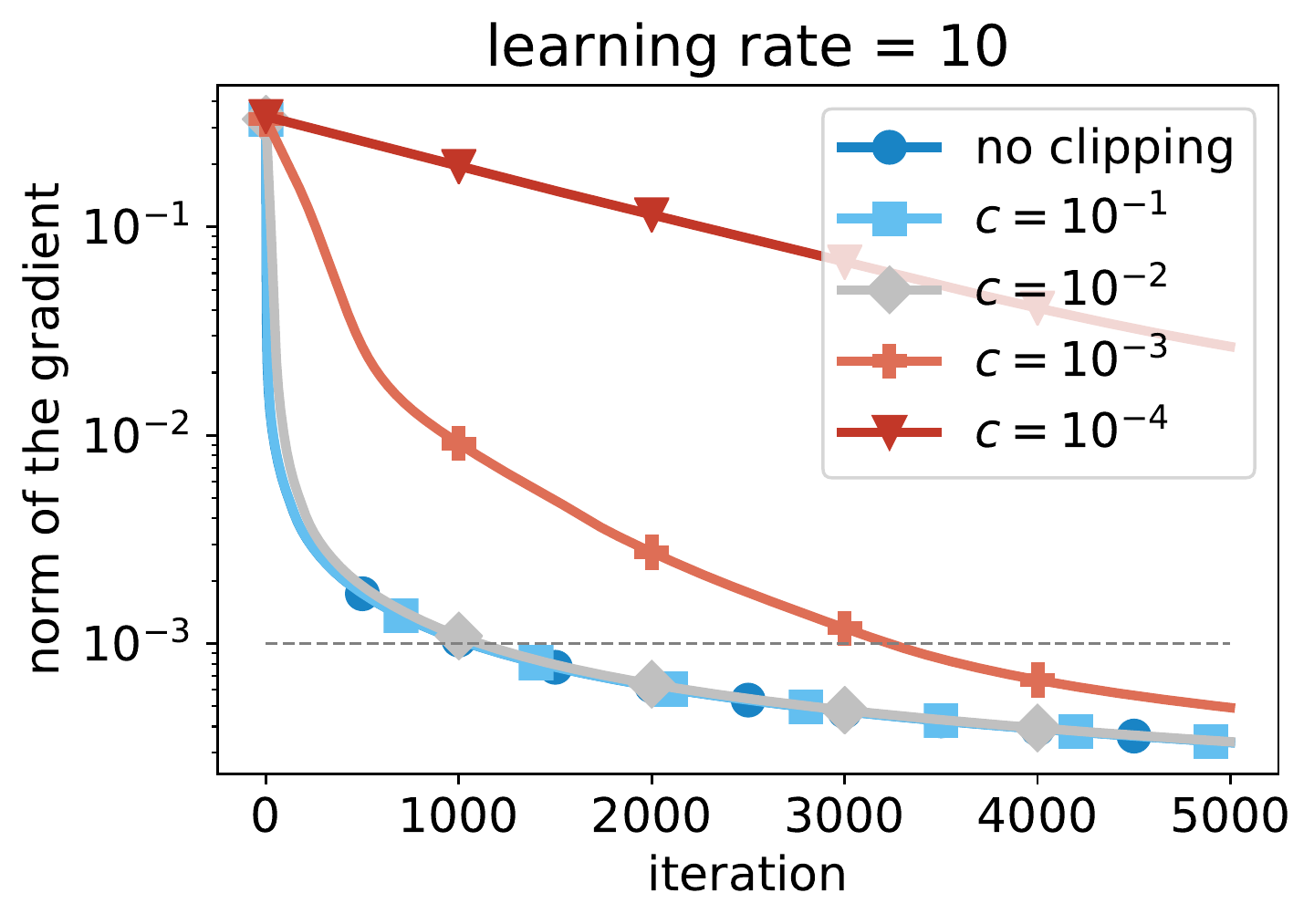}}%
	\subfigure[Tuned stepsize for target $\epsilon = 10^{-2}$ ]{\label{fig:det3}\hfill \includegraphics[width=0.32\linewidth]{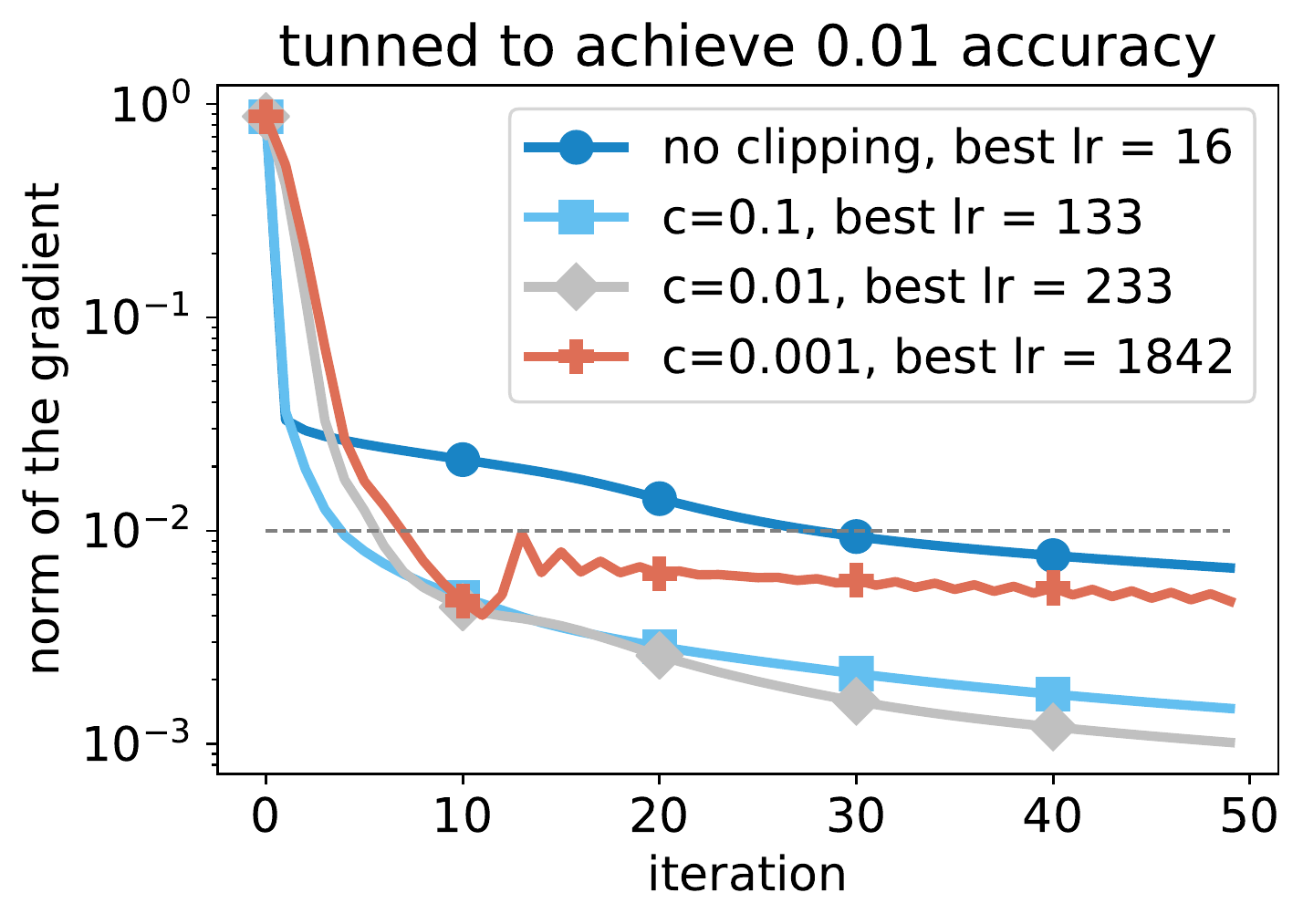}}
	\caption{Deterministic clipped gradient descent on the $\operatorname{w1a}$ dataset. We investigate the dependence of the convergence rate on the clipping parameter $c$. In Figures (a) and (b) we see that as soon as the clipping threshold is smaller or equal to the target gradient norm $\epsilon$, the convergence speed is affected only by a constant. In Figure (c), we see that as the clipping threshold $c$ decreases, the best tuned stepsize (tuned to reach $\epsilon = 10^{-2}$ fastest) decreases. These observations in accordance to the theory in Theorem~\ref{thm:det_cvx}. %
	}\label{fig:det}
\end{figure*}
We now introduce the central result of this paper: the convergence of clipped SGD that match the lower bounds above.%

\begin{theorem}\label{thm:csgd_large_c}
    If $f$ is \lzlone-smooth (but not necessarily convex) and we run clipped SGD for $T$ steps with step-size $\eta \leq 1 / [9(L_0 + c L_1)]$, then $\min_{t \in [0, T]} \E \norm{\nabla f(\xx_{t})}$ is upper bounded by
    \begin{align*}
        \cO\left(\min\Big\{\sigma,\frac{ \sigma^2}{c}\Big\}  + \sqrt{\eta (L_0 + c L_1)} \sigma + \sqrt{\frac{F_0}{\eta T}} + \frac{F_0}{\eta T c} \right)\,,
    \end{align*}
    where $F_0 = f(\xx_0) - f^\star$.
\end{theorem}

The convergence rate contains four terms: the first term does not decrease with neither the stepsize $\eta$ nor the number of iterations $T$ and it is due to unavoidable bias, as explained in the previous section. Due to Theorems~\ref{thm:lb_small_c},~\ref{thm:lb_large_c} this term is tight and cannot be improved. The second term is the stochastic noise term that decreases with the stepsize, and the last two terms are the optimization terms that describe how clipping affects convergence when the stochastic noise is zero ($\sigma = 0$), matching the convergence in Theorem~\ref{thm:det_nc}. Note that we precisely quantify the bias of clipped SGD under the general bounded variance assumption.

\paragraph{Comparison to the unclipped SGD.}
Under the standard smoothness Assumption~\ref{a:smooth-classic}, unclipped SGD requires the stepsize to be smaller than $\eta \leq L^{-1}$ and it converges at the rate \cite{Bottou18:sgd}
\begin{align*}
E_T \leq \cO\left(\sqrt{\eta L} \sigma + \sqrt{\frac{F_0}{\eta T}}\right) \,.
\end{align*}
where\footnote{Note that $E_T \geq \min_{t \in [0, T]} \E \norm{\nabla f(\xx_{t})}$.} $E_T := \left(\frac{1}{T} \sum_{t = 0}^T \| \nabla f(\xx_t)\|^2\right)^{\frac{1}{2}}$. 
In comparison to the unclipped SGD, clipped SGD \eqref{eq:clipped_sgd},
\begin{itemize}[leftmargin=12pt,nosep]
	\item Has an unavoidable bias term $\min\Big\{\sigma,\frac{ \sigma^2}{c}\Big\}$ that we discussed in detail in the previous Section~\ref{sec:unavoidable}.
	\item Has a smaller stochastic noise term, assuming that\footnote{We can always choose $L_0 = L$ and $L_1 = 0$ to satisfy this equation. Frequently, both $L_0$ and $L_1$ are much smaller than $L$ (see discussion in Section~\ref{sec:def}) } $L_0 + c L_1 \leq L$ .
	\item Similarly to the deterministic case (Thm.~\ref{thm:det_nc}), has an additional higher-order term $\nicefrac{F_0}{\eta T c}$. 
\end{itemize}

We want to highlight that the bias term $\min\Big\{\sigma,\frac{ \sigma^2}{c}\Big\} $ in our convergence rate is tight. This implies in particular that under general expected bounded noise Assumption~\ref{def:var}, clipped SGD cannot converge to the exact critical points of $f$, but the convergence neighborhood size decreases with increasing $c$. 

Similarly to \citet{zhang19:clippingL0-L1}, clipped SGD improves over unclipped SGD the dependence on the smoothness parameter from $L$ to $L_0$. In contrast to \citet{zhang19:clippingL0-L1} in our work we do not assume specific values of $c$, and use a weaker expected bounded noise Assumption~\ref{def:var}. %

The complete proof of Theorem~\ref{thm:det_scvx} can be found in Appendix~\ref{app:stoch_proofs}. We now give an intuitive proof sketch of Theorem~\ref{thm:csgd_large_c}. 
\vspace{-3mm}
\begin{proof}[Proof sketch (Theorem~\ref{thm:csgd_large_c})]
    The proof is split into two different cases, depending on how big $c$ is compared to $\sigma$.

    \textbf{Case $c \leq 4 \sigma$.} In this case, according to the lower bound in Theorem~\ref{thm:lb_small_c}, we can only show convergence of the gradient norm up to $\Theta(\sigma)$. To achieve this, we only need to consider the case when the gradients have $\norm{\nabla f(\xx_t)} \geq 6\sigma$, since otherwise the convergence to $\Theta(\sigma)$ is already achieved. We can show that in the case of large gradients (i.e.\ $\norm{\nabla f(\xx_t)} \geq 6\sigma$), the standard convergence results hold under uniformly bounded noise, because the gradient norm is large enough to compensate the (fixed) variance.

    \textbf{Case $c \geq 4\sigma$.} In this case, we analyze clipped SGD as some form of biased gradient descent. Note that under Uniform Boundedness (Assumption~\ref{def:unif_var}), the bias eventually vanishes for such large clipping thresholds. We precisely quantify the remaining bias term $B_t$ instead. After some manipulations, we obtain descent terms such as Equation~\eqref{eq:main_descent_bias} from Appendix~\ref{app:stoch_proofs}, and a bias term that writes as
    \begin{equation}
        B_t = \sqnorm{\esp{\clip(\nabla f_\xi(\xx_t))} - \clip(\nabla f(\xx_t))}.
    \end{equation}
    Using that the clipping operation is a projection on a convex set (on a ball of a radius $c$), then we can bound this term directly as $B_t \leq \sigma^2$. In particular, we can cancel it with descent terms when $\norm{\nabla f(\xx_t)}$ is large enough. 
    
    When $\norm{\nabla f(\xx)} \leq c/2$, then we can be more precise. In particular we can show that the probability that the stochastic gradient is clipped is smaller than $\sigma^2 / c^2$. Using this, we can refine the estimate of the bias as
    \begin{equation}
        B_t \leq 8 \frac{\sigma^4}{c^2} + 32\frac{\sigma^4}{c^4} \sqnorm{\nabla f(\xx)}.
    \end{equation} 
    The $\sigma^4 / c^2$ is the bias term that we find in the convergence rate, and the $\sqnorm{\nabla f(\xx)}$ term can be canceled with descent terms (that are also proportional to this) provided $\sigma^4 / c^4$ is small enough.
\end{proof}

\paragraph{Comparison to the prior works.} 
All of the prior works used stronger assumptions allowing for simplifications in their analysis, and allowing to mitigate the bias introduced by the clipping. 

For example, \cite{zhang19:clippingL0-L1}, \cite{zhang2020improved}, \cite{yang22:normalized_and_clipped_for_dp} considered large clipping thresholds ($c \geq \sigma$) and a stronger assumption of uniform boundness (Assumption~\ref{def:unif_var}), ensuring that the bias vanishes as we approach the optimum. 
The other prior work of \cite{gorbunov20:acc_clipping} used a specific clipping threshold $c$ and a specific large enough batch size allowing also to consider only one of the cases ($c \geq 4 \sigma$). They require the batch sizes to scale linearly with the number of iterations $T$, thus mitigating stochasticity in their gradients.
\cite{qian21:understanding_clipping} and \cite{chen20:geometric_clipping} impose some symmetry assumptions on the distribution of the stochastic gradients, which allows them to mitigate the bias introduced by the clipping operator. In the limit case of entirely symmetric distribution, the clipping operator does not change the direction of expected gradient at any point, thus allowing for the similar convergence analysis as with deterministic gradients.

\begin{figure*}[tb]
	\centering     %
	\subfigure[Quadratic function with $\chi^2$ stochastic noise]{
	\includegraphics[width=0.24\linewidth]{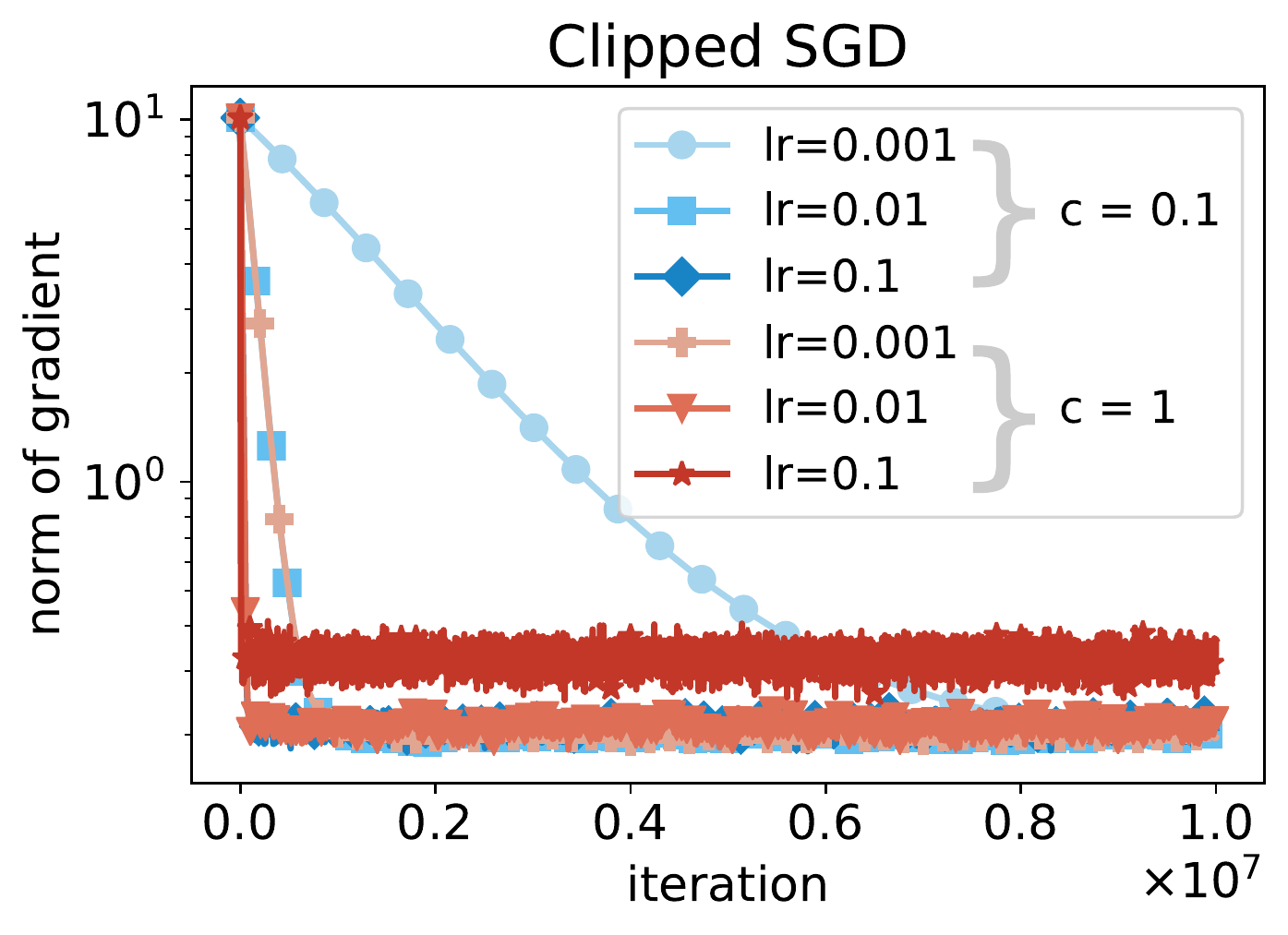}
	\includegraphics[width=0.24\linewidth]{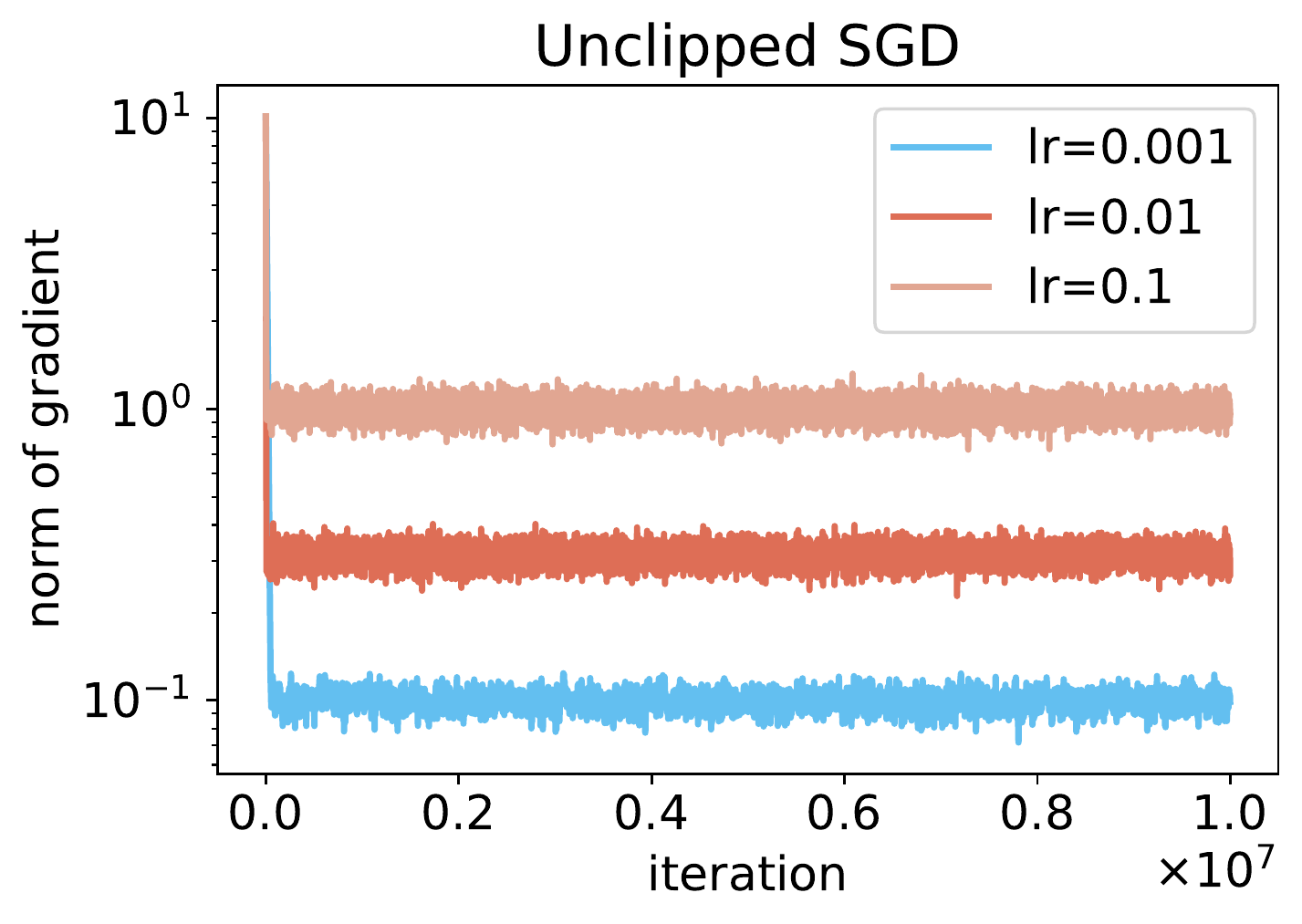}
}
	\subfigure[Logistic regression on $\operatorname{w1a}$ dataset (batch sitze = 1).]{
	\includegraphics[width=0.24\linewidth]{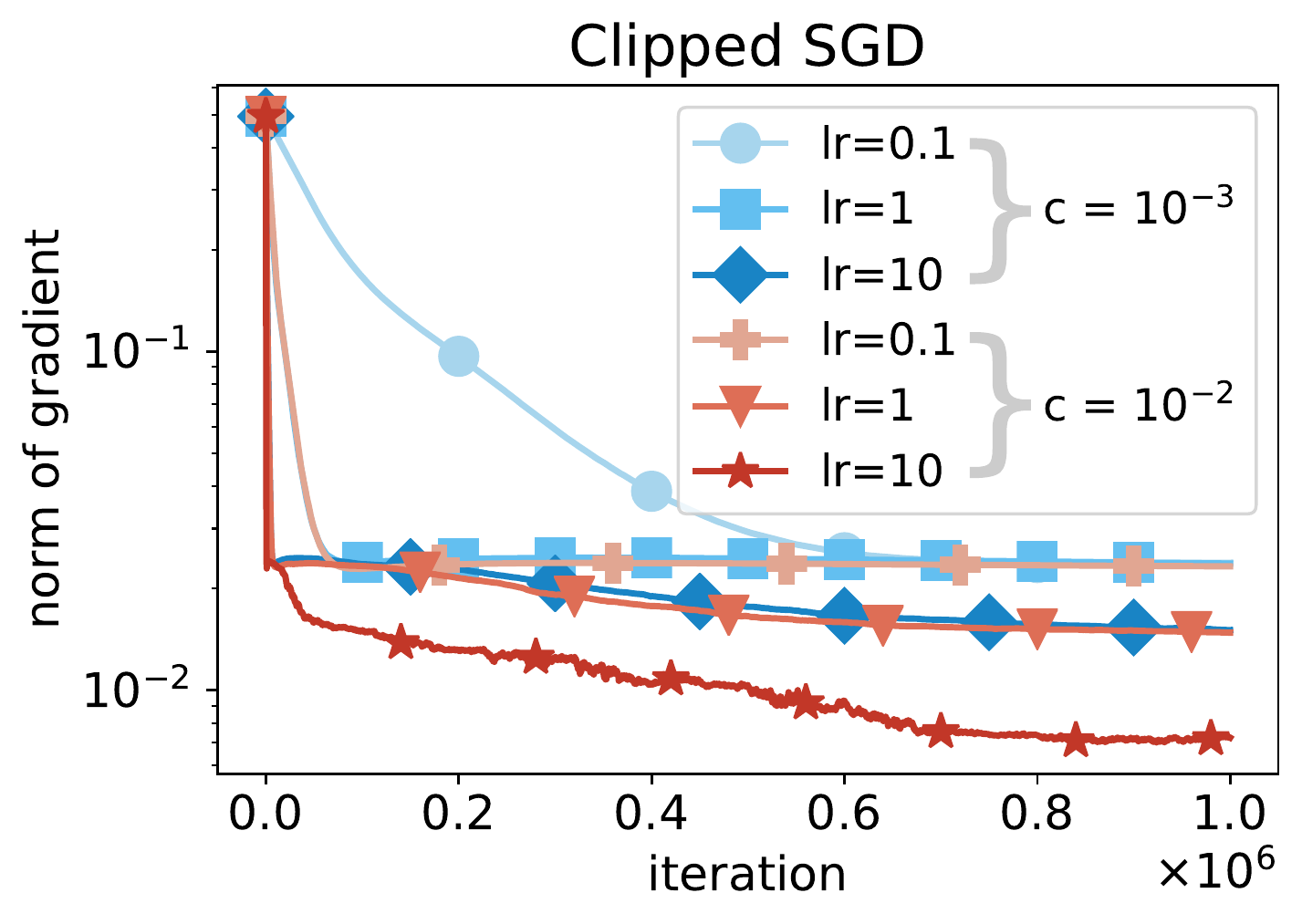}
	\includegraphics[width=0.24\linewidth]{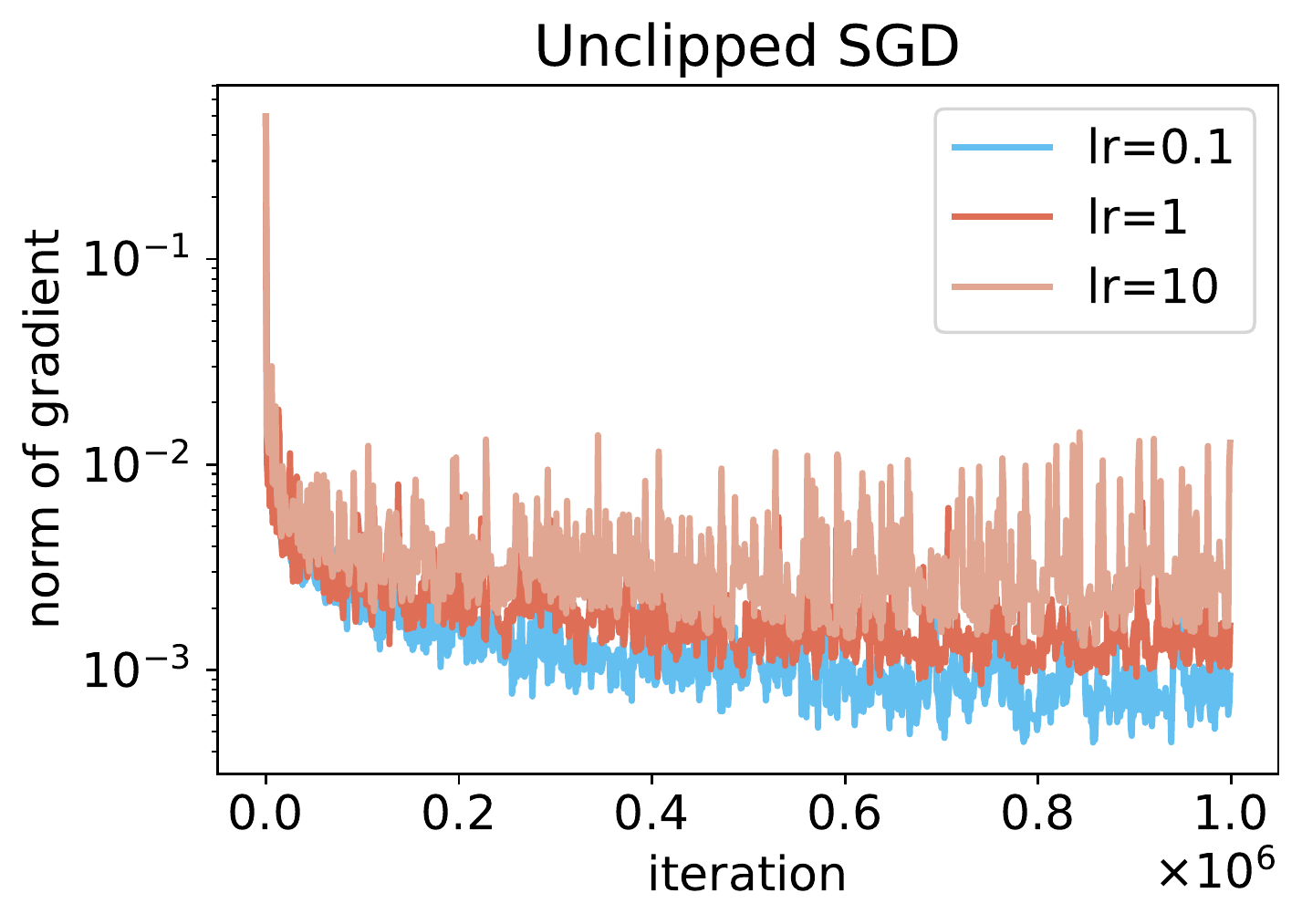}
}
	
	\caption{Stochastic gradient descent on a quadratic function with $\chi^2$ stochastic noise (left), and $\operatorname{w1a}$ dataset (right). Without clipping, decreasing the stepsize allows to achieve the smaller gradient norm. However, decreasing the stepsize with clipping does not allow to achieve better performance. This is because of the unavoidable bias term in Theorem~\ref{thm:csgd_large_c}. }\label{fig:stoch}
\end{figure*}

\subsection{Extension to differentially private SGD}

In differentially private SGD~\cite{abadi16:dp-sgd} every individual stochastic gradient in the batch is getting clipped individually before averaging the gradients over the batch, i.e.\  the algorithm is
\begin{align}\label{eq:dp-sgd}
\xx_{t+1} = \xx_t - \eta \left( \frac{1}{B} \sum_{i \in \cB_t} \clip(\nabla f_{\xi_i}(\xx_t)) + \zz_t\right),
\end{align}
where $\zz_t \sim \cN \left(0, \frac{\sigma_{\text{DP}}^2}{d} \mI \right)$ is the additional noise due to differential privacy. %

As detailed in Appendix~\ref{app:privacy} we can extend our analysis to this algorithm in a straightforward way and show that $T$ iterations of \eqref{eq:dp-sgd} allow to obtain gradient norm smaller than: 
\begin{align*}
\cO\bigg(\frac{L \eta }{c} \sigma_{\text{DP}}^2 + \sqrt{L \eta \sigma_{\text{DP}}}  &+ \min\bigg(\sigma^2,\frac{\sigma^4}{c^2}\bigg)  \\
&+ \eta L \frac{\sigma^2}{B}  + \frac{F_0}{\eta T} + \frac{F_0^2}{\eta^2 T^2 c^2}\bigg).
\end{align*}
where $B$ is the mini-batch size, and $L = L_0 + \max_t \norm{\nabla f(\xx_t)} L_1$, corresponding to the smoothness constant according to the standard $L$-smoothness assumption. 

Similarly to the clipped-SGD algorithm considered previously, DP-SGD also suffers from a bias term $\min\left(\sigma^2,\nicefrac{\sigma^4}{c^2}\right)$. Lower bounds in Theorems~\ref{thm:lb_small_c}, \ref{thm:lb_large_c} apply to DP-SGD, so this bias is also tight and unavoidable. 

In comparison to the clipped SGD \eqref{eq:clipped_sgd}, DP-SGD has additional terms related to the injected privacy noise $\sigma_{\text{DP}}$, and the stochastic noise (fourth term) is reduced by a factor $B$ due to mini-batching.

In order to have the formal privacy guarantees, one has to set the variance of additional DP noise appropriately, \citet{abadi16:dp-sgd} prove that for $\sigma_{\text{DP}} \geq \Omega \left( c d \frac{\sqrt{T\log \frac{1}{\delta}}}{\epsilon} \right)$ DP-SGD is $(\epsilon, \delta)$-differentially private.  

Related to prior work on DP-SGD that incorporate clipping in the convergence analysis \cite{chen20:geometric_clipping,yang22:normalized_and_clipped_for_dp}, our convergence rates are proven  only assuming bounded variance in expectation (Def.~\ref{def:var}) and without extra assumptions on the noise. They showcase the effect of the clipping threshold on the convergence of DP-SGD.

\section{Experiments}
In this section, we investigate the performance of gradient clipping on logistic regression on the \texttt{w1a} dataset \citep{w1a}, and on the artificial quadratic function $f(\xx) = \E_{\xi \sim \chi^2(1)}\left[f(\xx, \xi) := \frac{L}{2} \norm{\xx}^2 + \langle \xx, \xi\rangle\right]$, where $\xx\in\R^{100}$, we choose $L = 0.1$, and $\chi^2(1)$ is a (coordinate-wise) chi-squared distribution with $1$ degree of freedom. The goal is to highlight our theoretical results. 

\paragraph{Deterministic setting.} In the deterministic setting, one of our insights was that clipping does not degrade performance too much as long as the clipping threshold is bigger than the final target accuracy. We test this in Figures~\ref{fig:det1}, \ref{fig:det2} for logistic regression on \texttt{w1a} dataset, by plotting the clipped-GD with different values of $c$, for different target accuracies. We see that to reach accuracy $10^{-3}$, all values of $c$ (except from $c=10^{-4}$) perform relatively well. However, choosing $c=10^{-3}$ is not advisable if we only want to reach an error $\epsilon = 10^{-2}$, as can be seen in Figure~\ref{fig:det2} (note the different scaling of the $x$-axis in both plots).

In Figure~\ref{fig:det3}, we investigate the dependence between the clipping threshold $c$ and the step-size. We tune the stepsize separately for each clipping parameter $c$ over a logarithmic grid between $10^{-1}$ and $10^4$, ensuring that the optimal value is not at the edge of the grid. The best stepsize is selected as the one that reaches the target gradient norm $\epsilon = 10^{-2}$ the fastest. In Figure~\ref{fig:det3} we see that choosing smaller clipping radius allows for larger step-sizes, which speeds-up convergence overall.

In particular, we have verified that in the deterministic setting, clipping does not harm learning as long as the threshold is not too small compared to the target accuracy. Besides, clipping stabilizes learning, thus allowing for larger step-sizes, and thus faster convergence.

\paragraph{Stochastic setting.} We investigate clipped-SGD on both quadratic function with $\chi^2(1)$ stochastic noise, and logistic regression on \texttt{w1a} dataset. 
The results are plotted in Figure~\ref{fig:stoch}. They also verify the theory, since larger learning rates are always better when using clipping (compared to unclipped). It is also interesting to note that the curves are determined by the product $c \eta$, which is how large the step is when clipping happens. However, we see that in this case, with such small values of $c$, clipped-SGD does not quite reach the performance of vanilla SGD.

\section{Conclusion}
In this paper, we have rigorously analyzed gradient clipping, both in the deterministic setting and under standard noise assumptions. While previous works focus on exact convergence under strong assumptions (in particular, often for a fixed clipping threshold), we tightly characterized (with both upper and lower bounds) the bias introduced by clipped-SGD for any clipping threshold. 

Our work paves the way for better understanding clipping when used with other algorithms, such as accelerated or momentum methods or FedAvg. 
In particular, it can lead to an improved analysis of privacy guarantees in applications that rely on clipped SGD as an underlying black box%
, on the one hand for existing, but also future applications. %

\section*{Acknowledgments}
The authors would like to thank Ryan McKenna and Martin Jaggi for useful discussions. 
We also thank anonymous reviewers for their valuable comments.
SS acknowledges partial funding from a Meta Privacy Enhancing Technologies Research Award 2022. %
AK acknowledges funding from a Google PhD Fellowship.

\clearpage

{\small
\bibliography{reference}
\bibliographystyle{icml2023}
}

\newpage
\appendix
\onecolumn

\section{Implications of \lzlone smoothness}

\begin{lemma}\label{lem:L-smooth-impl}
	If Assumption~\ref{a:smooth} holds, then it also holds that
	\begin{align}
	f(\yy) - f(\xx) &\leq \nabla f(\xx)^\top (\yy - \xx) + \frac{(L_0 + \norm{\nabla f(\xx)} L_1)}{2} \norm{\xx - \yy}^2\,, && \forall \xx, \yy \in \R^d \text{ with } \norm{\xx-\yy} \leq \frac{1}{L_1}\,. \label{eq:1515}
	\end{align}
\end{lemma}
For the proof see \cite{zhang2020improved}, Appendix A.1.

\begin{lemma}\label{lem:L-smooth-conv-impl}
	If Assumption~\ref{a:smooth} holds, then it also holds that
	\begin{align*}
	\norm{\nabla f(\xx)}^2 \leq 2(L_0 +  L_1 \norm{\nabla f(\xx)})\left(f(\xx) - f^\star\right) && \forall \xx \in \R^d  \,,
	\end{align*}
	where $f^\star = \inf_{\xx}f(\xx)$.
\end{lemma}
\begin{proof}[Proof of Lemma~\ref{lem:L-smooth-conv-impl}]
We start the proof by applying the previous Lemma~\ref{lem:L-smooth-impl} for $\yy= \xx - \frac{1}{L_0 + \norm{\nabla f(\xx)} L_1} \nabla f(\xx)$.
Note that $\norm{\xx-\yy} = \frac{\norm{\nabla f(\xx)}}{L_0 + \norm{\nabla f(\xx)} L_1} \leq \frac{1}{L_1}$ and we can apply the inequality:
\begin{align*}
 f^\star \leq f\left(\xx - \frac{1}{L_0 + \norm{\nabla f(\xx)} L_1} \nabla f(\xx)\right) \stackrel{\eqref{eq:1515}}{\leq} f(\xx) - \frac{1}{2(L_0 + \norm{\nabla f(\xx)} L_1)} \norm{\nabla f(\xx)}^2 \,,
\end{align*}
and rearranging gives us the desired property.
\end{proof}

\section{Deterministic proofs}
This section contains the main proofs from the paper. We skip the non-convex proof, since it will be a direct consequence of the stochastic result. %

\subsection{Convex case (Theorem~\ref{thm:det_cvx})}
Defining $\alpha_t = \min\{1, \frac{c}{\norm{\nabla f(\xx_t)}}\}$ we have:
\begin{align*}
\norm{\xx_{t+1} - \xx^\star}^2 &\leq \norm{\xx_{t} - \xx^\star - \eta \alpha_t \nabla f(\xx_{t})  }^2 = \norm{\xx_{t} - \xx^\star}^2 + \eta^2 \alpha_t^2 \norm{\nabla f(\xx_{t})}^2 - 2 \alpha_t \eta \langle \nabla f(\xx_{t}) , \xx_{t} - \xx^\star \rangle \\
&\leq \norm{\xx_{t} - \xx^\star}^2 + \eta^2 \alpha_t^2 \norm{\nabla f(\xx_{t})}^2 -2 \eta \alpha_t \left( f(\xx_{t})  - f^\star \right) \,.
\end{align*}
We consider two cases: when clipping happens, and when clipping does not happen. 

\paragraph{Case 1:} $\alpha_t = 1$, meaning that $ \norm{\nabla f(\xx_{t})} \leq c$. Then
\begin{align*}
\norm{\xx_{t+1} - \xx^\star}^2 \leq \norm{\xx_{t} - \xx^\star}^2 + \eta^2 \norm{\nabla f(\xx_{t})}^2 -2 \eta \left( f(\xx_{t})  - f^\star \right) \,.
\end{align*}
Using the implication of \lzlone smoothness and convexity in Lemma~\ref{lem:L-smooth-conv-impl}, 
\begin{align*}
\norm{\nabla f(\xx_{t})}^2 \leq 2(L_0 +  L_1 \norm{\nabla f(\xx_t)})\left(f(\xx_{t}) - f^\star\right) \leq  2(L_0 +  L_1 c)\left(f(\xx_{t}) - f^\star\right) \,.
\end{align*}
Further,
\begin{align*}
\norm{\xx_{t+1} - \xx^\star}^2 \leq \norm{\xx_{t} - \xx^\star}^2 + 2 (L_0 + L_1 c) \eta^2 \left( f(\xx_{t})  - f^\star \right) -2 \eta \left( f(\xx_{t})  - f^\star \right) \,,
\end{align*}
and by setting $\eta \leq \frac{1}{2(L_0 + L_1 c)}$ we obtain
\begin{align*}
\norm{\xx_{t+1} - \xx^\star}^2 \leq \norm{\xx_{t} - \xx^\star}^2  - \eta \left( f(\xx_{t})  - f^\star \right) \,.
\end{align*}

\paragraph{Case 2:}$\alpha_t = \frac{c}{\norm{\nabla f(\xx_{t})}}$, meaning that $ \norm{\nabla f(\xx_{t})} > c$. Then,
\begin{align*}
\norm{\xx_{t+1} - \xx^\star}^2 &\leq \norm{\xx_{t} - \xx^\star}^2 + \eta^2 c^2  -2 \eta \frac{c}{\norm{\nabla f(\xx_{t})}}  \left( f(\xx_{t})  - f^\star \right) \,.
\end{align*}
If it holds that $\eta^2 c^2 \leq \eta \frac{c}{\norm{\nabla f(\xx_{t})}}  \left( f(\xx_{t})  - f^\star \right)$, then we will get %
\begin{align}\label{eq:descent_second_case}
\norm{\xx_{t+1} - \xx^\star}^2 &\leq \norm{\xx_{t} - \xx^\star}^2 - \eta \frac{c}{\sqrt{2L}} \sqrt{ \left( f(\xx_{t})  - f^\star \right)} \,.
\end{align}
Lets now see under which stepsizes the condition $\eta \leq \frac{1}{c \norm{\nabla f(\xx_{t})}}  \left( f(\xx_{t})  - f^\star \right)$ holds by upper bounding the rhs. By \lzlone smoothness (and Lemma~\ref{lem:L-smooth-conv-impl}) we know that $(f(\xx_{t})  - f^\star) \geq \frac{\norm{\nabla f(\xx_t)}^2}{2 (L_0 + L_1 \norm{\nabla f(\xx_t)})}$ and thus
\begin{align*}
\frac{1}{c \norm{\nabla f(\xx_{t})}}  \left( f(\xx_{t})  - f^\star \right) \geq \frac{1}{2 (L_0 \frac{c}{\norm{\nabla f(\xx_t)}} + L_1 c)} \geq \frac{1}{2 (L_0 + L_1 c)} \,,
\end{align*}
where the last inequality is because $\frac{c}{\norm{\nabla f(\xx_t)}} \leq 1$ by our assumptions on $\alpha_t$ in this case. This means that using stepsize $\eta \leq \frac{1}{2 (L_0 + L_1 c)}$, it will hold that $\eta \leq \frac{1}{c \norm{\nabla f(\xx_{t})}}  \left( f(\xx_{t})  - f^\star \right)$ and thus \eqref{eq:descent_second_case} will hold.

\paragraph{Summing the two cases.} We define $\cT_1$ the set of iterations when clipping does not happen and $\cT_2$ as set of iterations when clipping happens. Taking the average over $T + 1$ iterations
\begin{align*}
\frac{1}{T + 1} \sum_{t \in \cT_1} (f(\xx_{t}) - f^\star )+ \frac{1}{T + 1} \sum_{t \in \cT_2} \frac{c}{\sqrt{2L}}\sqrt{f(\xx_{t}) - f^\star } \leq \frac{\norm{\xx_{0} - \xx^\star}^2}{\eta(T + 1)} \,.
\end{align*}
This means that both (i)
\begin{align*}
\frac{1}{T + 1} \sum_{t \in \cT_1} (f(\xx_{t}) - f^\star ) \leq \frac{\norm{\xx_{0} - \xx^\star}^2}{\eta (T + 1)}  \,,
\end{align*}
and (ii)
\begin{align*}
\frac{1}{T + 1} \sum_{t \in \cT_2} \sqrt{f(\xx_{t}) - f^\star} \leq \frac{\norm{\xx_{0} - \xx^\star}^2 \sqrt{2L} }{\eta c (T + 1)} \,.
\end{align*}
For the first inequality (i) using that $x^2 \geq 2 \epsilon x - \epsilon^2$ for any $\epsilon, x>0$, and defining for simplicity $A := \frac{\norm{\xx_{0} - \xx^\star}^2}{\eta (T + 1)} $ we get
\begin{align*}
\frac{1}{ T + 1}\sum_{t \in \cT_1}\left( 2 \epsilon \sqrt{f(\xx_{t}) - f^\star} - \epsilon^2 \right)\leq A \,,
\end{align*}
and thus, 
\begin{align*}
\frac{1}{ T + 1}\sum_{t \in \cT_1} \sqrt{f(\xx_{t}) - f^\star} \leq \frac{A}{2 \epsilon} + \frac{\epsilon}{2} \,.
\end{align*}
Choosing $\epsilon = \sqrt{A}$, we get
\begin{align*}
\frac{1}{ T + 1}\sum_{t \in \cT_1} \sqrt{f(\xx_{t}) - f^\star} \leq \sqrt{A} \leq \sqrt{\frac{\norm{\xx_0 - \xx^\star}^2}{{\eta (T + 1)}}} \,.
\end{align*}

This implies that
\begin{align*}
\frac{1}{T + 1} \sum_{t = 0}^{T}  \sqrt{f(\xx_{t}) - f^\star } \leq \sqrt{\frac{R_0^2}{\eta (T + 1)}} + \frac{R_0^2 \sqrt{2L}}{\eta c (T + 1)} \,.
\end{align*}
We further use that $f(\xx_{t + 1})\leq f(\xx_{t})$ and get a last-iterate convergence rate
\begin{align*}
\sqrt{f(\xx_{T}) - f^\star } \leq \sqrt{\frac{R_0^2}{\eta (T + 1)}} + \frac{R_0^2 \sqrt{2L}}{\eta c (T + 1)} \,.
\end{align*}
Squaring both of the sides, and using that $(a + b)^2 \leq 2 a^2 + 2 b^2 ~\forall a, b$, we get
\begin{align*}
f(\xx_{T}) - f^\star \leq \frac{2 R_0^2}{\eta (T + 1)} + \frac{4 L R_0^4 }{\eta^2 c^2 (T + 1)^2} \,.
\end{align*}

\subsection{Strongly convex case (Theorem~\ref{thm:det_scvx})}

\paragraph{Recursive argument.}
First, since the strongly convex function is also convex, we can apply the result of the previous theorem here to get
\begin{align*}
f(\xx_{T}) - f^\star \leq \frac{2 R_0^2}{\eta T} + \frac{4 L R_0^4 }{\eta^2 c^2 T^2} \,.
\end{align*}
We remind that $R_0 = \norm{\xx_{0} - \xx^\star}$. Using strong-convexity, we also know that 
\begin{align*}
f(\xx_{T}) - f^\star \geq \frac{\mu}{2} R_t^2 \,,
\end{align*}
Thus,
\begin{align*}
R_t^2 \leq \frac{4 R_0^2}{\mu \eta T} + \frac{8 L R_0^4}{\mu \eta^2 c^2 T^2} \,.
\end{align*}
Thus, to get $R_t^2 \leq \frac{R_0^2}{2}$, it is enough to take $t \geq \max\{\frac{16}{\mu \eta}, \frac{6 R_0 \sqrt{L}}{\eta c \sqrt{\mu}}\}$ (as both terms become less that $R_0^2/4$).

Repeating this argument, 
we can see the iteration complexity can be bounded by
\[
 T= \cO \left( \frac{1}{\mu \eta} \log \left( \frac{R_0^2}{\epsilon} \right) +  \frac{R_0 \sqrt{L}}{\eta c \sqrt{\mu }} \right) \,.
\]

\paragraph{Small gradients.} Let us start again from the convex bound (Theorem~\ref{thm:det_cvx}). Now, we will instead use that: 
\[
\sqnorm{\nabla f (\xx_t)} \leq 2L\left(f(\xx_t)-f^\star \right) \leq 2L\frac{R_0^2}{\eta t}\left(1 + \frac{R_0^2 L}{c^2 \eta t}\right)\,.
\]
Now introduce $t_0$ which is such that:
\begin{equation}
    t_0 = \frac{8L R_0^2}{\eta c^2}, 
\end{equation}
then we have that for all $t \geq t_0$:
\begin{equation}
    \norm{\nabla f(\xx_{t})} \leq c.
\end{equation}
In particular, we know that no clipping happens after $t_0$, and so we obtain the standard linear convergence rate, so that the final convergence time is: 
\begin{equation}
    T = O\left(\frac{1}{\eta \mu} \log\left(\frac{R_0^2}{\epsilon}\right) + \frac{L R_0^2}{\eta c^2}\right) \,.
\end{equation}

\paragraph{Comparing the two rates.} Note that no rate is better than the other, and we can use one or the other depending on the relationship between $c$ and $\sqrt{L \mu}R_0$. %

\section{Stochastic proofs.}\label{app:stoch_proofs}

We now proceed to the proof of Theorem~\ref{thm:csgd_large_c}. The proof will be in two parts: we will first prove convergence up to $\sigma^2$, and then refine this for large values of $c$.

\subsection{Preliminaries}

We now state a very simple lemma, which is direct but at the core of our decomposition, and so we highlight it here. 

\begin{lemma} \label{lemma:decomposition}
    For any $\alpha > 0$ and $\uu \in \R^d$, the following holds:
    \begin{equation} \label{eq:descent_term}
        - \nabla f(\xx)^\top \uu = - \frac{\alpha}{2}\sqnorm{\nabla f(\xx)} - \frac{1}{2\alpha}\sqnorm{\uu} + \frac{1}{2\alpha}\sqnorm{\uu - \alpha \nabla f(\xx)} \,.
    \end{equation}        
\end{lemma}

\subsection{First part of the proof: convergence up to $\sigma$ (small $c$)}

In this section for simplicity we assume that $c < 4 \sigma$ and prove that the gradient norm converges up to a level $\sigma$. Note that this assumption on $c$ is not restrictive since the case $c > 4 \sigma$ is covered by the other part of the proof, in which we show better convergence to $\cO(\frac{\sigma^2}{c})$.

\paragraph{Large gradients.} Let us start by assuming that $\norm{\nabla f(\xx_t)} \geq 6 \sigma$. Note that numerical constant is (relatively) arbitrary and could be tightened, but we choose it high to keep the proof clean and simple.  

We start the analysis by using \lzlone smoothness property from Lemma~\ref{lem:L-smooth-impl}. Note that for any stepsize $\eta < \frac{1}{L_0 + cL_1}$ it holds that  $\| \xx_{t + 1}- \xx_t \| = \eta \|\gg(\xx_t) \| \leq \eta c \leq \frac{1}{L_1}$
\begin{align}
    f(\xx_{t+1}) - f(\xx_t) &\leq - \eta \nabla f(\xx_t)^\top \gg(\xx_t) + \frac{\eta^2 (L_0 + \norm{\nabla f(\xx_t)} L_1)}{2} \sqnorm{\gg(\xx_t)}\nonumber\\
    &\leq - \eta \nabla f(\xx_t)^\top \gg(\xx_t) + \frac{\eta^2 (L_0 + \norm{\nabla f(\xx_t)} L_1)}{2}c^2\nonumber\\
    & \leq - \eta \nabla f(\xx_t)^\top \gg(\xx_t) + \frac{\eta^2 (L_0 + c L_1)}{2} c \norm{\nabla f(\xx_t)}\label{eq:main_small_c} \,,
\end{align}
where the last inequality is because we assumed that $c \leq 4 \sigma \leq \norm{\nabla f(\xx_t)}$. 

\paragraph{Uniformly bounded variance, def. \ref{def:unif_var}.} In this case, let us first assume that strong variance holds with constant $3$, \emph{i.e.}, that $\norm{\nabla f_\xi(\xx_t) - \nabla f(\xx_t)} \leq 3\sigma$ with probability one. In this case, we can write, where $\alpha_\xi = \min\left(1, c / \norm{\nabla f_\xi(\xx_t)}\right)$: 
\begin{align*}
    - \nabla f(\xx_t)^\top \gg(\xx_t) &= - \alpha_\xi \sqnorm{\nabla f(\xx_t)} - \alpha_\xi \nabla f(\xx_t)^\top \left( \nabla f_\xi(\xx_t) - \nabla f(\xx_t)\right)\\
    &\leq - \alpha_\xi \sqnorm{\nabla f(\xx_t)} + \alpha_\xi \norm{\nabla f(\xx_t)} \norm{\nabla f_\xi(\xx_t) - \nabla f(\xx_t)}\\
    &\leq - \alpha_\xi \sqnorm{\nabla f(\xx_t)} + 3\alpha_\xi \norm{\nabla f(\xx_t)} \sigma \\
    &\leq - \frac{\alpha_\xi}{2} \sqnorm{\nabla f(\xx_t)},
\end{align*}
where the last line follows from the fact that $\sigma < \norm{\nabla f(\xx_t)} / 6$. In particular, using the strong variance assumption, we know that $\norm{\nabla f_\xi(\xx_t)} \leq 2\norm{\nabla f(\xx_t)}$, so that $\alpha_\xi \geq \min(1, c / (2\norm{\nabla f(\xx_t)})) \geq c / (2\norm{\nabla f(\xx_t)})$. In particular: 
\begin{equation}
    - \nabla f(\xx_t)^\top \nabla f_\xi(\xx_t) \leq - \frac{c}{4} \norm{\nabla f(\xx_t)}.
\end{equation}
Then, we can plug this into Equation~\eqref{eq:main_small_c}, which leads to: 
\begin{equation}
    \esp{f(\xx_{t+1})} - f(\xx_t) \leq - \frac{\eta c}{4}\left(1 - 2 \eta (L_0 + c L_1)\right) \norm{\nabla f(\xx_t)} \,.
\end{equation}
In particular, choosing $\eta \leq  \left(4[L_0 + c L_1]\right)^{-1}$, we obtain: 
\begin{equation}
    \frac{\eta c}{8} \norm{\nabla f(\xx_t)} \leq f(\xx_t) - f(\xx_{t+1}) .
\end{equation}

\paragraph{Bounded variance in expectation, Def~\ref{def:var}.} In this case, we cannot write the same inequalities as before with probability $1$. However, we can still guarantee the bound with large enough probability. We define $\delta = \mathds{1}\{\norm{\nabla f_\xi(\xx) - \nabla f(\xx)} > 3\sigma\}$. We will use conditional expectations to write
\begin{align*}
    \esp{- \alpha_\xi \nabla f(\xx)^\top \nabla f_\xi(\xx)} \leq p(\delta = 0)\underbrace{\mathbb{E}\left[- \alpha_\xi \nabla f(\xx)^\top \nabla f_\xi(\xx) | \delta = 0\right]}_{:=T_1} + p(\delta = 1)\underbrace{\mathbb{E}\left[- \alpha_\xi \nabla f(\xx)^\top \nabla f_\xi(\xx) | \delta = 1\right]}_{:=T_2} \,.
\end{align*}
We bound the first term $T_1$ the same way as in previous case of uniformly bounded noise. For the second term, by Cauchy-Schwartz inequality, and defining $\alpha = \min\left(1, c / \norm{\nabla f(\xx)}\right)$ we write
\begin{align*}
    T_2 = \mathbb{E}\left[- \alpha_\xi \nabla f(\xx)^\top \nabla f_\xi(\xx) | \delta = 1\right] &\leq \norm{\nabla f(\xx)}\mathbb{E}\left[\norm{\alpha_\xi \nabla f_\xi(\xx)}| \delta = 1\right] \leq \alpha \sqnorm{\nabla f(\xx)} \,,
\end{align*}
where the last inequality is because we assumed that the full gradients are large $\norm{\nabla f(\xx)} > 6 \sigma$, but the clipping threshold is small $c \leq 4 \sigma$. Thus, the full gradients would always get clipped, and $\norm{\alpha \nabla f(\xx)} = c \geq \norm{\alpha_\xi \nabla f_\xi(\xx)}$. We remind that $\alpha_\xi = \min\left(1, c / \norm{\nabla f_\xi(\xx)}\right)$.%

Now, it just remains to bound $p(\delta = 1)$. Using Markov inequality, we have that: 
\begin{equation}
    p(\delta = 1) = p(\norm{\nabla f_\xi(\xx) - \nabla f(\xx)}^2 > 9\sigma^2) \leq 1/9.
\end{equation}
Similarly, $p(\delta = 0) = 1 - p(\delta = 1) \geq 8/9$. In the end, we obtain that:
\begin{equation}
    - \esp{\nabla f(\xx)^\top \Gxi} \leq - c \left(\frac{1}{4} \times \frac{8}{9} - \frac{1}{9}\right) \norm{\nabla f(\xx)} = - \frac{c}{9} \norm{\nabla f(\xx)} \,.
\end{equation}
We further plug the result into \eqref{eq:main_small_c}, and obtain
\begin{equation}
    \esp{f(\xx_{t+1})} - f(\xx_t) \leq - \frac{\eta c}{9}\left(1 - \frac{9 \eta}{2}(L_0 + c L_1)\right) \norm{\nabla f(\xx_t)},
\end{equation}
and so with $\eta \leq \left(9[L_0 + c L_1]\right)^{-1}$, we obtain: 
\begin{equation}\label{eq:proof-step}
    \esp{f(\xx_{t+1})} - f(\xx_t) \leq - \frac{\eta c}{18}\norm{\nabla f(\xx_t)}.
\end{equation}

\paragraph{Final convergence.}  If for at least one iteration $t$ it happens that the gradient norm is small $\norm{\nabla f(\xx_t)} \leq 6 \sigma$, then it simply holds that 
\begin{align*}
\min_{t\in[1, T]}\E \sqnorm{\nabla f(\xx_{t})} \leq O\left(\sigma^2\right).
\end{align*}
Otherwise, for all $t$ iterations the gradient norm is large $\norm{\nabla f(\xx_t)} > 6 \sigma$ and thus \eqref{eq:proof-step} holds for all the iterations. Averaging over $1 \leq t \leq  T + 1$, we obtain
\begin{equation}
    \frac{1}{T+ 1} \sum_{t=0}^T \norm{\nabla f(\xx_t)} \leq  \cO\left(\frac{f(\xx_0) - f^\star}{\eta c T}\right),
\end{equation}
Combining these two cases we conclude that 
\begin{align*}
\min_{t\in[1, T]}\E \sqnorm{\nabla f(\xx_{t})} \leq \cO\left(\sigma^2 +  \frac{f(\xx_0) - f^\star}{\eta c T}\right),
\end{align*}

\subsection{Second part of the proof: convergence up to $\sigma^2 / c$ (large $c$).}

In this second part we assume that the clipping radius is large, $c \geq 4 \sigma$.
Although, the algorithm \eqref{eq:clipped_sgd} clips the stochastic gradients $\nabla f_\xi(\xx_t)$, for the proof we will consider the two cases based on the full gradient $\nabla f(\xx_t)$: when the full gradient $\nabla f(\xx_t)$ is clipped and when it is not clipped. 

Similarly to previous case, we start by using \lzlone smoothness
\begin{equation} \label{eq:main_large_c}
    f(\xx_{t+1}) - f(\xx_t) \leq - \eta \nabla f(\xx_t)^\top \gg(\xx_t) + \frac{\eta^2 (L_0 + \norm{\nabla f(\xx_t)} L_1)}{2} \sqnorm{\gg(\xx_t)} \,.
\end{equation}

\paragraph{First case, full gradient is clipped $ \norm{\nabla f(\xx_t)} > c$.}

In this case, we use \eqref{eq:descent_term} with $\alpha = \frac{c}{\norm{\nabla f(\xx_t)}}$ and $\uu = \gg(\xx_t)$. Since $\alpha \nabla f(\xx_t) = \clip(\nabla f(\xx_t))$, this leads to %
\begin{equation} \label{eq:main_descent_bias}
     - \nabla f(\xx_t)^\top \gg(\xx_t) = - \frac{c}{2}\norm{\nabla f(\xx_t)} - \frac{1}{2\alpha}\sqnorm{\gg(\xx_t)} + \frac{1}{2\alpha}\sqnorm{\gg(\xx_t) - \clip(\nabla f(\xx_t))} \,.
\end{equation}
We now use that $\gg(\xx_t) = \clip(\nabla f_\xi(\xx_t))$, and use that clipping is a projection on onto a convex set (ball of radius $c$), and thus is Lipshitz operator with Lipshitz constant $1$, we write 
\begin{align*}
      - \nabla f(\xx_t)^\top \E \gg(\xx_t) &\leq - \frac{c}{2}\norm{\nabla f(\xx_t)} - \frac{1}{2\alpha}\esp{\sqnorm{\gg(\xx_t)}} + \frac{1}{2\alpha} \E \sqnorm{\nabla f_\xi(\xx_t) - \nabla f(\xx_t)}\\ 
      &\leq - \frac{c}{2}\norm{\nabla f(\xx_t)} - \frac{1}{2\alpha}\esp{\sqnorm{\gg(\xx_t)}} + \frac{\sigma^2}{2c} \norm{\nabla f(\xx_t)}\\ 
     &=  - \frac{1}{2\alpha}\esp{\sqnorm{\gg(\xx_t)}} - \frac{c}{2}\norm{\nabla f(\xx_t)} \left(1 - \frac{\sigma^2}{c^2}\right)\\
     &\leq  - \frac{\norm{\nabla f(\xx_t)}}{2c}\esp{\sqnorm{\gg(\xx_t)}} - \frac{c}{4}\norm{\nabla f(\xx_t)}, 
\end{align*}
where in the last line we used that $\nicefrac{\sigma^2}{c^2} \leq \nicefrac{1}{2}$ and $\alpha \leq 1$. Plugging this into \eqref{eq:main_large_c} we get
\begin{align*}
    \esp{f(\xx_{t+1})} - f(\xx_t) &\leq - \frac{\eta\norm{\nabla f(\xx_t)}}{2c}\esp{\sqnorm{\gg(\xx_t)}}  - \frac{\eta c}{4}\norm{\nabla f(\xx_t)} + \frac{\eta^2 (L_0 + \norm{\nabla f(\xx_t)} L_1)}{2} \esp{\sqnorm{\gg(\xx_t)}}\\
    &= - \frac{\eta c}{4}\norm{\nabla f(\xx_t)} - \frac{\eta\norm{\nabla f(\xx_t)}}{2c}\esp{\sqnorm{\gg(\xx_t)}} \left(1 - \eta c L_1\right) + \frac{\eta^2 L_0}{2} \esp{\sqnorm{\gg(\xx_t)}}\\
    &\leq - \frac{\eta c}{4}\norm{\nabla f(\xx_t)} -\frac{\eta}{2}\esp{\sqnorm{\gg(\xx_t)}} \left(1 - \eta c L_1\right) + \frac{\eta^2 L_0}{2} \esp{\sqnorm{\gg(\xx_t)}}\\
    &= - \frac{\eta c}{4}\norm{\nabla f(\xx_t)} -\frac{\eta}{2}\esp{\sqnorm{\gg(\xx_t)}} \left(1 - \eta [L_0 + c L_1]\right).
\end{align*}
In particular, choosing $\eta \leq (L_0 + c L_1)^{-1}$, we obtain: 
\begin{align}\label{eq:first_case}
    \frac{c}{4} \norm{\nabla f(\xx_t)} \leq \frac{f(\xx_t) - \E f(\xx_{t+1})}{\eta} \,.
\end{align}

Note that we do not obtain variance terms, but similarly to the previous section it is because we have assumed that the norm of the gradient is larger than $\sigma$, then the noise term can be hidden in the gradient norm term.

\paragraph{Second case, $ c > \norm{\nabla f(\xx_t)} > \nicefrac{c}{2}$.} The proof follows very closely the previous case with the difference that the full gradient $\nabla f(\xx_t)$ is not clipped. 
We use Equation~\eqref{eq:descent_term} with $\alpha = 1$. This leads to %
\begin{align*}
     - \nabla f(\xx_t)^\top \E \gg(\xx_t) &= - \frac{1}{2}\sqnorm{\nabla f(\xx_t)} - \frac{1}{2} \E \sqnorm{\gg(\xx_t)} + \frac{1}{2}\E \sqnorm{ \gg(\xx_t) - \nabla f(\xx_t)}\\
     &\leq - \frac{1}{2}\sqnorm{\nabla f(\xx_t)} - \frac{1}{2}\E \sqnorm{\gg(\xx_t)} + \frac{\sigma^2}{2} \,,
\end{align*}
where on the last line we used that clipping is Lipshitz operator with constant $1$, as it is a projection on a convex set. 
We now use that $- \norm{\nabla f(\xx_t)} \leq - c / 2$ for the first term and $1 \leq \norm{\nabla f(\xx_t)} / c$ for the last term:
\begin{align*}
     - \nabla f(\xx_t)^\top \E\gg(\xx_t) &\leq - \frac{1}{2}\E \sqnorm{\gg(\xx_t)} - \frac{c}{4}\norm{\nabla f(\xx_t)} + \frac{\sigma^2}{2c} \norm{\nabla f(\xx_t)}\\
     &\leq - \frac{1}{2}\E\sqnorm{\gg(\xx_t)} - \frac{c}{4}\norm{\nabla f(\xx_t)} \left(1 - 2\frac{\sigma^2}{c^2}\right)\\
     &\leq - \frac{1}{2}\E \sqnorm{\gg(\xx_t)} - \frac{c}{8}\norm{\nabla f(\xx_t)},
\end{align*}
where in the last line we used that $\sigma^2 / c^2 \leq 1/4$. Similarly to the previous case, we plug it into \eqref{eq:main_large_c} and use that $-1 \leq - \frac{\norm{\nabla f(\xx_t)}}{c}$ we have that
\begin{align*}
    \E f(\xx_{t+1}) - f(\xx_t) &\leq - \frac{\eta \norm{\nabla f(\xx_t)}}{2 c}\esp{\sqnorm{\gg(\xx_t)}} - \frac{c \eta }{8}\norm{\nabla f(\xx_t)} + \frac{\eta^2 (L_0 + \norm{\nabla f(\xx_t)} L_1)}{2} \esp{\sqnorm{\gg(\xx_t)}} \\
    &\leq  - \frac{c \eta }{8}\norm{\nabla f(\xx_t)} -  \frac{\eta}{2} \esp{\sqnorm{\gg(\xx_t)}} \left( \frac{\norm{\nabla f(\xx_t)}}{c} (1 - \eta c L_1)- \eta L_0 \right)\\
    &{\leq} - \frac{c \eta }{8}\norm{\nabla f(\xx_t)} - \frac{\eta}{2} \esp{\norm{\gg(\xx_t)}^2} \left(\frac{1}{2} - \eta [L_0 + c L_1]\right)
\end{align*}
where on the last line we used that $\norm{\nabla f(\xx_t)} > c/2$. Using that $\eta \leq \frac{1}{2}(L_0 + c L_1)^{-1}$
\begin{align}\label{eq:second_case}
    \frac{c}{8} \norm{\nabla f(\xx_t)} \leq \frac{f(\xx_t) - \E f(\xx_{t+1})}{\eta}  \,.
\end{align}

\paragraph{Third case, $\norm{\nabla f(\xx_t)} < \nicefrac{c}{2}$.}
In this case, we do not have convergence to the exact optimum. 

We start by defining $\delta_t = \mathds{1}\{\norm{\nabla f_\xi(\xx_t)} > c\}$ is the indicator function that at time step $t$ the stochastic gradient is getting clipped. We will start by showing that $\E \delta_t \leq \frac{4 \sigma^2}{c^2}$. 
\begin{align*}
\E \delta_t = \Pr[\delta_t = 1] = \Pr\left[\norm{\nabla f_\xi(\xx_t)} > c\right] \leq \Pr\left[\norm{\nabla f_\xi(\xx_t) - \nabla f(\xx_t)} > \frac{c}{2}\right] \leq \frac{4 \sigma^2}{c^2} \,,
\end{align*}
where the last inequality is due to Markov's inequality. The first inequality is because $\norm{\nabla f_\xi(\xx_t)} \leq \norm{\nabla f_\xi(\xx_t) - \nabla f(\xx_t)} + \norm{\nabla f(\xx_t)} \leq \norm{\nabla f_\xi(\xx_t) - \nabla f(\xx_t)} +\frac{c}{2}$.

Now that we have $\E \delta_t \leq \frac{4 \sigma^2}{c^2}$ we can use it to bound the difference $\sqnorm{\nabla f(\xx_t) - \E\gg(\xx_t)} $. In particular, since $\delta_t$ takes values $0$ and $1$, we have that $\esp{\delta_t} = p(\delta_t = 1)$ and so $\esp{\delta_t X} = \esp{\delta_t} \esp{X | \delta_t}$ for any random variable $X$.
\begin{align}
\sqnorm{\nabla f(\xx_t) - \E\gg(\xx_t)} = \norm{\E \left(1 - \frac{c}{\norm{\nabla f_\xi(\xx_t)}}\right)\nabla f_\xi(\xx_t) \delta_t}^2\nonumber = \esp{\delta_t}^2 \norm{\esp{\left(1 - \frac{c}{\norm{\nabla f_\xi(\xx_t)}}\right) \nabla f_\xi(\xx_t) | \delta_t = 1}}^2\nonumber \,.
\end{align}
At this point, we use Jensen inequality on the conditional expectation (since all terms are positive and the squared norm is a convex function) and get that: 
\begin{align}
\sqnorm{\nabla f(\xx_t) - \E\gg(\xx_t)} &\leq \esp{\delta_t}^2 \esp{\left(1 - \frac{c}{\norm{\nabla f_\xi(\xx_t)}}\right)^2 \norm{\nabla f_\xi(\xx_t)}^2 | \delta_t = 1}\nonumber\\
&\leq \esp{\delta_t}^2 \esp{\norm{\nabla f_\xi(\xx_t)}^2 | \delta_t = 1 }\nonumber\\
&\leq 2 \esp{\delta_t}^2 \esp{\norm{\nabla f_\xi(\xx_t) - \nabla f(\xx_t)}^2 | \delta_t = 1} + 2 \esp{\delta_t}^2 \esp{\norm{\nabla f(\xx_t)}^2 | \delta_t = 1}\nonumber\\
& \leq 2 \esp{\delta_t} \esp{\norm{\nabla f_\xi(\xx_t) - \nabla f(\xx_t)}^2} + 2\esp{\delta_t}^2 \norm{\nabla f(\xx_t)}^2\nonumber\\
& \leq \frac{8 \sigma^4}{c^2}  + \frac{32 \sigma^4}{c^4}\norm{\nabla f(\xx_t)}^2,\label{eq:bias_main}
\end{align}
where on the second line we used that $\left(1 - \frac{c}{\norm{\nabla f_\xi(\xx_t)}}\right)^2 \leq 1$ when $\delta_t = 1$, and on the last line that $\esp{\delta_t} \leq 4\sigma^2 / c^2$. We further use \eqref{eq:descent_term} with $\alpha = 1$ and $\uu = \E\gg(\xx_t)$, we get
\begin{align*}
- \nabla f(\xx)^\top \E\gg(\xx_t) &= - \frac{1}{2}\sqnorm{\nabla f(\xx)} - \frac{1}{2}\sqnorm{\E\gg(\xx_t)} + \frac{1}{2}\sqnorm{\E\gg(\xx_t) - \nabla f(\xx)}\\
&\leq - \frac{1}{2}\sqnorm{\nabla f(\xx)} - \frac{1}{2}\sqnorm{\E\gg(\xx_t)} +\frac{4 \sigma^4}{c^2}  + \frac{4 \sigma^2}{c^2}\norm{\nabla f(\xx_t)}^2\\
&\stackrel{\sigma \leq \frac{c}{4}}{\leq} - \frac{1}{4}\sqnorm{\nabla f(\xx)} - \frac{1}{2}\sqnorm{\E\gg(\xx_t)} +\frac{4 \sigma^4}{c^2}  \,.
\end{align*}

Plugging this into \eqref{eq:main_large_c}, for $\eta \leq \frac{1}{8(L_0 + cL_1)}$, we get by dropping the $\sqnorm{\E\gg(\xx_t)}$ term and using that $\norm{\nabla f(\xx_t)} \leq c$ that: 
\begin{align}
\E f(\xx_{t+1}) - f(\xx_t) &\leq - \frac{\eta}{4}\sqnorm{\nabla f(\xx)} - \frac{\eta}{2}\sqnorm{\E\gg(\xx_t)} +\frac{4 \eta \sigma^4}{c^2} + \frac{\eta^2 (L_0 + \norm{\nabla f(\xx_t)} L_1)}{2} \E \sqnorm{\gg(\xx_t)}\label{eq:start_difference}\\
&\leq - \frac{\eta}{4}\sqnorm{\nabla f(\xx)} +\frac{4 \eta \sigma^4}{c^2} + \frac{\eta^2 (L_0 + c L_1)}{2} \E \sqnorm{\gg(\xx_t) - \nabla f(\xx_t) + \nabla f(\xx_t)}\nonumber\\
&\leq - \frac{\eta}{4}\sqnorm{\nabla f(\xx)} +\frac{4 \eta\sigma^4}{c^2} + \eta^2 (L_0 +c L_1) \E \sqnorm{\gg(\xx_t) - \nabla f(\xx_t)} + \eta^2 (L_0 + c L_1) \norm{\nabla f(\xx_t)}^2\nonumber\\
&\stackrel{\eta \leq \frac{1}{8(L_0 + cL_1)}}{\leq} - \frac{\eta}{8}\sqnorm{\nabla f(\xx)} +\frac{4 \eta\sigma^4}{c^2} + \eta^2 (L_0 + c L_1) \E \sqnorm{\gg(\xx_t) - \nabla f(\xx_t)}. \nonumber
\end{align}
Note that clipping is the orthogonal projection onto the ball of radius $c$, which we denote ${\rm proj}_c$ and $\norm{\nabla f(x_t)} \leq c$, so it is not affected by the projection. In particular:
\begin{equation}
    \E\sqnorm{g(x_t) - \nabla f(x_t)} = \E\sqnorm{{\rm proj}_c(\nabla f_\xi (x_t)) - {\rm proj}_c(\nabla f(x_t))} \leq \E \sqnorm{\nabla f_\xi (x_t) - \nabla f(x_t)} \leq \sigma^2,
\end{equation}
and we thus get
\begin{equation}
    \E f(\xx_{t+1}) - f(\xx_t) \leq - \frac{\eta}{8}\sqnorm{\nabla f(\xx)} +\frac{4 \eta\sigma^4}{c^2} + \eta^2 (L_0 + c L_1) \sigma^2,
\end{equation}
and so: 
\begin{align}\label{eq:third_case_final}
    \frac{1}{8} \sqnorm{\nabla f(\xx_t)} \leq \frac{f(\xx_t) - \E f(\xx_{t+1})}{\eta} + \eta (L_0 + c L_1) \sigma^2 + \frac{4 \sigma^4}{c^2}.
\end{align}
In particular, we have:
\begin{itemize}
    \item One variance term that fades with the step-size.
    \item One bias term that remains even for very small step-sizes. 
\end{itemize}

\paragraph{Wrapping up. } We now combine the three cases above. Defining $\cT_1$ is the set of indices with $\norm{\nabla f(\xx_t)} \geq \frac{c}{2}$ (we note that this covers the first and the second cases from above, but both of them leads to the same final inequality \eqref{eq:second_case}), and $\cT_2$ is the set of indices with $\norm{\nabla f(\xx_t)} < \frac{c}{2}$, this inequality \eqref{eq:third_case_final} holds. Summing up over all the indices $1\leq t\leq T + 1$, we get
\begin{align*}
\frac{1}{ 8 (T + 1)} \left(\sum_{t\in \cT_1} c \E\norm{\nabla f(\xx_t)} + \sum_{t \in \cT_2} \E \sqnorm{\nabla f(\xx_t)}\right) \leq \frac{f(\xx_0) - f^\star}{\eta (T  +1) } + \eta (L_0 + c L_1) \sigma^2 + \frac{4 \sigma^4}{c^2} \,.
\end{align*}
This means that both (i)
\begin{align*}
\frac{1}{ 8 (T + 1)} \sum_{t\in \cT_1} c \E\norm{\nabla f(\xx_t)} \leq \frac{f(\xx_0) - f^\star}{\eta (T + 1) } + \eta (L_0 + c L_1) \sigma^2 + \frac{4 \sigma^4}{c^2} \,,
\end{align*}
and (ii)
\begin{align*}
\frac{1}{ 8 (T + 1)}\sum_{t \in \cT_2} \E \sqnorm{\nabla f(\xx_t)} \leq \frac{f(\xx_0) - f^\star}{\eta (T + 1) } + \eta (L_0 + c L_1) \sigma^2 + \frac{4 \sigma^4}{c^2} \,,
\end{align*}
for the last inequality using that $x^2 \geq 2 \epsilon x - \epsilon^2$ for any $\epsilon, x>0$, and defining for simplicity $A := 8 \frac{f(\xx_0) - f^\star}{\eta T } +8  \eta (L_0 + c L_1) \sigma^2 + \frac{32 \sigma^4}{c^2}$ we get
\begin{align*}
\frac{1}{ T + 1}\sum_{t \in \cT_2}\left( 2 \epsilon\E \norm{\nabla f(\xx_t)}  - \epsilon^2 \right)\leq A \,,
\end{align*}
and thus, 
\begin{align*}
\frac{1}{ T + 1}\sum_{t \in \cT_2} \E \norm{\nabla f(\xx_t)}\leq \frac{A}{2 \epsilon} + \frac{\epsilon}{2} \,.
\end{align*}
Choosing $\epsilon = \sqrt{A}$, we get
\begin{align*}
\frac{1}{ T + 1}\sum_{t \in \cT_2} \E \norm{\nabla f(\xx_t)}\leq \sqrt{A} \leq \sqrt{8 \frac{f(\xx_0) - f^\star}{\eta (T + 1) }} + \sqrt{8  \eta (L_0 + c L_1) \sigma^2 } + \sqrt{\frac{32 \sigma^4}{c^2}} \,.
\end{align*}
Summing up the two cases again, and using that $\frac{\sigma }{c} \leq \frac{1}{4}$ we get
\begin{align*}
\frac{1}{ T + 1}\sum_{t =0}^T \E \norm{\nabla f(\xx_t)} \leq \cO\left( \sqrt{\frac{f(\xx_0) - f^\star}{\eta T }}  +\frac{f(\xx_0) - f^\star}{\eta c T } + \sqrt{ \eta (L_0 + c L_1)} \sigma+ \frac{\sigma^2}{c} \right) \,.
\end{align*}

\subsection{Differentially Private SGD}\label{app:privacy}

\subsubsection{Modification to the proof to include mini-batches}
Using $\gg(\xx_t) = \frac{1}{B} \sum_{\xi \in \cB_t} \clip (\nabla f_{\xi}(\xx_t))$, the proof is exactly the same as in the previous case, with the only difference in the case where $c \geq 4 \sigma$ and small gradients (third case) $\norm{\nabla f(\xx_t)} < \frac{c}{2}$. Starting with equation~\eqref{eq:start_difference}, we obtain:

\begin{align*}
\E f(\xx_{t+1}) - f(\xx_t) &\leq - \frac{\eta}{4}\sqnorm{\nabla f(\xx)} - \frac{\eta}{2}\sqnorm{\E\gg(\xx_t)} +\frac{4 \eta \sigma^4}{c^2} + \frac{\eta^2 (L_0 + \norm{\nabla f(\xx_t)} L_1)}{2} \E \sqnorm{\gg(\xx_t)}\\
&\leq - \frac{\eta}{4}\sqnorm{\nabla f(\xx)} - \frac{\eta}{2}\sqnorm{\E\gg(\xx_t)}  +\frac{4 \eta\sigma^4}{c^2} + \frac{\eta^2 (L_0 +c L_1)}{2} \E \sqnorm{\gg(\xx_t) - \E \gg(\xx_t)}\\
&\qquad  + \frac{\eta^2 (L_0 + c L_1)}{2} \norm{\E \gg(\xx_t)}^2\nonumber \,.
\end{align*}
We now estimate the term variance term  $\E \sqnorm{\gg(\xx_t) - \E \gg(\xx_t)}$ more tightly in order to get the variance reduction due to the batch size $B$.

\begin{align*}
\E \sqnorm{\gg(\xx_t) - \E \gg(\xx_t)}&= \E \sqnorm{\frac{1}{B} \sum_{i \in \cB_t} \clip (\nabla f_{\xi_i}(\xx_t)) - \E \gg(\xx_t)} = \frac{1}{B^2} \sum_{i \in \cB_t}  \E\norm{\clip (\nabla f_{\xi_i}(\xx_t)) - \E \gg(\xx_t)}^2 \\
&\leq \frac{1}{B^2} \sum_{i \in \cB_t}  2 \E\norm{\clip (\nabla f_{\xi_i}(\xx_t)) - \nabla f(\xx_t)}^2  + \frac{2}{B} \norm{\nabla f(\xx_t)- \E \gg(\xx_t)}^2 \\
&\stackrel{\eqref{eq:bias_main}}{\leq} \frac{2 \sigma^2}{B}  + \frac{2}{B} \left[\frac{8 \sigma^4}{c^2}  + \frac{8 \sigma^2}{c^2}\norm{\nabla f(\xx_t)}^2\right]\\
& \leq \frac{2 \sigma^2}{B}  + \frac{2}{B} \left[\frac{\sigma^2}{2}  + 2 \sigma^2\right] \leq 6 \frac{\sigma^2}{B} \,,
\end{align*}
where we used that $\norm{\nabla f(\xx_t)}\leq \frac{c}{2}$ and that $\sigma \leq \frac{c}{4}$. The rest of the proof is exactly the same as before, by substituting now the $\sigma^2$ term with $\frac{\sigma^2}{B}$, we would arrive at the convergence rate of 
\begin{align*}
\frac{1}{ T}\sum_{t =0}^T \E \norm{\nabla f(\xx_t)} \leq \cO\left( \sqrt{\frac{f(\xx_0) - f^\star}{\eta T }}  +\frac{f(\xx_0) - f^\star}{\eta c T } + \sqrt{ \eta (L_0 + c L_1)} \frac{\sigma}{\sqrt{B}}+ \frac{\sigma^2}{c} \right) \,.
\end{align*}

\subsubsection{Modification to the proof to include stochastic noises}
The gradients applied in DP-SGD \eqref{eq:dp-sgd} have the form $\gg(\xx_t) + \zz_t$, where $\zz_t$ is a Gaussian noise with variance $\sigma_{\text{DP}}$. In order to add this additional Gaussian noise, we would need to modify the first step of the proof, that is using \lzlone smoothness  %
\begin{align*}
\E f(\xx_{t+1}) - f(\xx_t) &\leq - \eta \nabla f(\xx_t)^\top \gg(\xx_t) + \frac{\eta^2 (L_0 + \norm{\nabla f(\xx_t)} L_1)}{2} \sqnorm{\gg(\xx_t)}\nonumber + \frac{\eta^2 (L_0 + \norm{\nabla f(\xx_t)} L_1)}{2} \sigma_{\text{DP}}^2 \,.
\end{align*}
The rest of the proof remains the same, with having an additional $\sigma_{\text{DP}}^2$ term in the convergence. We thus would arrive to the following convergence rate where for simplicity we define $L = L_0 + \max_t \norm{\nabla f(\xx_t)} L_1$
\begin{align*}
O\bigg( \frac{L \eta }{c} \sigma_{\text{DP}}^2 + \sqrt{L \eta \sigma_{\text{DP}}} +  \min\bigg(\sigma,\frac{ \sigma^2}{c}\bigg)  + \sqrt{\eta L} \frac{\sigma}{\sqrt{B}} + \sqrt{\frac{F_0}{\eta T}} &+ \frac{F_0}{\eta T c} \bigg).
\end{align*}

\subsection{Lower bound}
\label{app:lower_bound}
We now prove the lower bound. 
\begin{proof}[Proofs of Theorems~\ref{thm:lb_small_c} and~\ref{thm:lb_large_c}.]
    Let us consider the simple noise $a \cB(p)$, where $a> 0$ and $\cB(p)$ is a Bernoulli random variable with mean $p \leq 1/2$. Consider a function such that the stochastic gradients are of the form:
    \begin{equation}
        \nabla f_\xi(x) = x + a \cB(p)     \,.
    \end{equation}
    Now consider $x = - p c / (1 - p)$. The stochastic gradient at $x$ when the Bernoulli is $0$ is not clipped, since $|x| = pc / (1 - p) \leq c$. Yet, the stochastic gradient for positive values of the Bernoulli random variable is:
    \begin{equation}
        \nabla f_a(x) = -pc / (1 - p) + a \geq a - c \geq c.
    \end{equation}
    In particular, we have that: 
    \begin{equation}
        \esp{\clip(\nabla f_\xi(x))} = (1 - p)x + pc = (1 - p) \times (-pc) / (1 - p) + p c = 0.
    \end{equation}
    Let us now evaluate $\nabla f(x)$. We have:
    \begin{equation}
        \nabla f(x) = x + pa = p\left(a - \frac{c}{1-p}\right).
    \end{equation}

\paragraph{Small $c$.} Now fix a clipping radius $c$, such that $c \leq 2\sigma$, and take $a = 4\sigma$. We choose $p(1-p) = 1/16$, so that $p = (2 - \sqrt{3})/4 \leq 1/4$. In this case,
\begin{equation}
        \nabla f(x) = p\left(a - \frac{c}{1-p}\right) \geq p\left(4 \sigma - 2 \sigma \times \frac{4}{3}\right) \geq \frac{(2 - \sqrt{3})\sigma}{3} \geq \frac{\sigma}{12}.
    \end{equation}

\paragraph{Large $c$.} Now fix a clipping radius $c$, such that $c \geq \sigma$ and $c \leq a / 2$. To ensure that the noise has variance $\sigma^2$, $p$ has to be such that:
    \begin{equation}
        p(1 - p) = \sigma^2 / a^2 \leq 1/16 \,.
    \end{equation}
    Thus, we have that $p \leq 1/4$ (since we chose $p<1/2$). In particular, also using that $c \leq a/2$:
    \begin{equation}
        \nabla f(x) = \frac{\sigma^2}{a^2(1-p)}\left(a - \frac{c}{1-p}\right) \geq \frac{\sigma^2}{3 a(1-p)} \geq \frac{\sigma^2}{3 a} \,.
    \end{equation}
    It now remains to choose $a = 2c$ (which satisfies all previous conditions), and we obtain:
    \begin{equation}
        \nabla f(x) \geq \frac{\sigma^2}{6 c} \,.
    \end{equation}
\end{proof}

\end{document}